\PassOptionsToPackage{table}{xcolor}
\documentclass[10pt,twocolumn,letterpaper]{article}
\usepackage[pagenumbers]{cvpr}
\def\confName{CVPR}
\def\confYear{2027}

\usepackage{amsmath,amsfonts,bm}

\def\eqref#1{equation~\ref{#1}}

\def\1{\bm{1}}

\DeclareMathAlphabet{\mathsfit}{\encodingdefault}{\sfdefault}{m}{sl}
\SetMathAlphabet{\mathsfit}{bold}{\encodingdefault}{\sfdefault}{bx}{n}

\usepackage{titletoc,microtype,multirow,tabularx,adjustbox,makecell}
\usepackage{mathtools,amsthm,comment,tcolorbox,pifont}
\usepackage{algorithm,algorithmic}
\definecolor{cvprblue}{rgb}{0.21,0.49,0.74}
\usepackage[pagebackref,breaklinks,colorlinks,allcolors=cvprblue]{hyperref}
\usepackage[capitalize,noabbrev]{cleveref}
\newcommand{\MLoRAAuthors}{Junkang Liu\\Tianjin University\\{\tt\small junkangliukk@gmail.com}}
\theoremstyle{plain}
\newtheorem{theorem}{Theorem}[section]
\newtheorem{proposition}[theorem]{Proposition}
\newtheorem{lemma}[theorem]{Lemma}
\newtheorem{corollary}[theorem]{Corollary}
\theoremstyle{definition}

\theoremstyle{remark}

\definecolor{LightBlue}{RGB}{135,190,255}

\newtcolorbox{bluebox}[1][]{
	colback=LightBlue!5!white,
	colframe=LightBlue,
	boxrule=0.7pt,
	arc=2mm,
	left=0mm,
	right=0mm,
	top=0mm,
	bottom=0mm,
	halign=center,
	#1
}

\title{MLoRA: Continual Low-Rank Adaptation via Relative-Response Subspace Selection and Momentum Orthogonalization}
\author{\MLoRAAuthors}
\begin{document}
\maketitle

\begin{abstract}
	Low-rank adaptation (LoRA) provides a parameter-efficient
	approach to continual learning, but its nominal rank can
	conceal a loss of effective adaptation capacity.
	We identify \emph{spectral plasticity collapse}:
	during sequential adaptation, update energy becomes
	concentrated in a small subset of singular modes,
	leaving much of the available low-rank space underutilized.
	This exposes a limitation of interference avoidance alone:
	protecting historical representations does not ensure
	that the remaining adaptation capacity is responsive
	to new tasks or effectively utilized.
	To address this problem, we propose \texttt{MLoRA},
	which jointly controls capacity allocation and utilization.
	First, historical whitening identifies input directions
	with strong current-task response relative to accumulated
	historical response, yielding a task-adaptive basis
	that remains fixed during training.
	Second, approximate polar orthogonalization of momentum
	updates reduces spectral concentration within the
	selected space.
	An orthonormal basis connects these mechanisms by
	transferring the factor-update spectrum exactly to
	the induced weight update.
	We establish a max--min characterization of exact
	subspace selection and derive cumulative spectral
	bounds under controlled cross-step anisotropy.
	Across five class-incremental benchmarks and eight
	incremental settings, \texttt{MLoRA} achieves the highest
	mean accuracy in 15 of 16 reported metrics.
\end{abstract}

\section{Introduction}
\label{sec:introduction}

Continual learning (CL) requires models to acquire knowledge
from sequential tasks while retaining performance on
previous ones~\cite{EwC,lorapmNIPS25,2021nullspace}.
With the emergence of foundation models, this challenge
increasingly concerns adapting pretrained representations
under limited computation and memory
budgets~\cite{2021clip,cheng2024disentangled}.
Low-Rank Adaptation (LoRA)~\cite{lora} provides an efficient
parameterization by expressing each task-specific weight
update as $\Delta W=BA$ while freezing the backbone
within each task.
However, a fixed adaptation rank does not ensure that
all available directions are effectively used.
An update can retain its algebraic rank while concentrating
most of its energy in only a few singular modes.
This distinction matters in continual adaptation,
where a limited low-rank space must support learning
across successive tasks.

We observe a phenomenon that exposes this limitation:
\emph{\textbf{ spectral plasticity collapse}}.
During continual LoRA adaptation, the learned update
$\Delta W$ can exhibit a highly concentrated spectrum,
leaving many available modes weakly utilized
(Figure~\ref{fig:sp_curve} and Figure~\ref{fig:rank}).
We examine this behavior through four complementary
diagnostics: the singular-value distribution,
effective rank, the nonzero condition number $\kappa_+$,
and stable rank.
The former reveals how update energy is distributed
across modes, while the latter quantifies the effective
number of modes carrying that energy.
Here, collapse refers to the concentration of update energy in a few
singular directions, leading to lower effective and stable ranks
without necessarily reducing the algebraic rank.
Thus, maintaining the nominal LoRA rank does not
guarantee effective use of its adaptation capacity.

Muon \cite{jordan2024muon} offers a natural starting point for addressing
this spectral concentration.
Through approximate polar orthogonalization of momentum,
it reshapes the update spectrum toward balanced nonzero
singular values \cite{bernstein2025derivingmuon}.
This motivates using spectral balancing to improve
update utilization within a fixed-rank adapter.
However, spectral balancing of LoRA factors does not
automatically translate into spectral balancing of
the weight update.
When both factors are trained, their changes induce
$\delta W=\delta B A+B\delta A+\delta B\delta A$.
Even if the factor steps are individually balanced,
multiplication by the factors and interaction between
terms can produce an uneven weight-step spectrum.
The challenge is to make Muon's spectral
control carry  to update of the model.

We address this issue by fixing a row-orthonormal
basis $A$ and updating only $B$ with Muon.
Since $AA^\top=I_r$, each weight step
$\delta W=\delta BA$ has the same nonzero singular
values as $\delta B$.
Muon’s spectral balancing therefore transfers directly
to the weight step.
Likewise, $BA$ preserves the nonzero singular values
of $B$, so the basis introduces no spectral distortion.

Fixing $A$ makes the choice of its subspace important,
because the model can only update within that space.
An orthonormal basis preserves the update spectrum,
but its directions may not be useful for the new task.
Methods such as InfLoRA~\cite{infloraCVPR24} avoid
a retained historical feature subspace to reduce interference.
However, shared directions can also support new learning,
and directions with little historical response may still
be unhelpful for the current task.
We therefore select directions with strong current-task
response relative to their historical response,
providing a task-aware space for spectral balancing.

These observations suggest two connected requirements.
The adaptation space should contain directions with
strong current-task response relative to historical
exposure, and its parameterization should preserve
spectral control from factor space to weight space.
Together, they lead to following question:

\begin{bluebox}
	\textit{
		How can Muon's spectral balancing be transferred
		to LoRA weight updates within a subspace suited
		to continual adaptation?
	}
\end{bluebox}

\begin{figure*}[t]
	\centering
	\captionsetup[subfigure]{font=small,labelfont=bf}
	\begin{subfigure}[t]{0.24\linewidth}
		\centering
		\includegraphics[width=\linewidth]
		{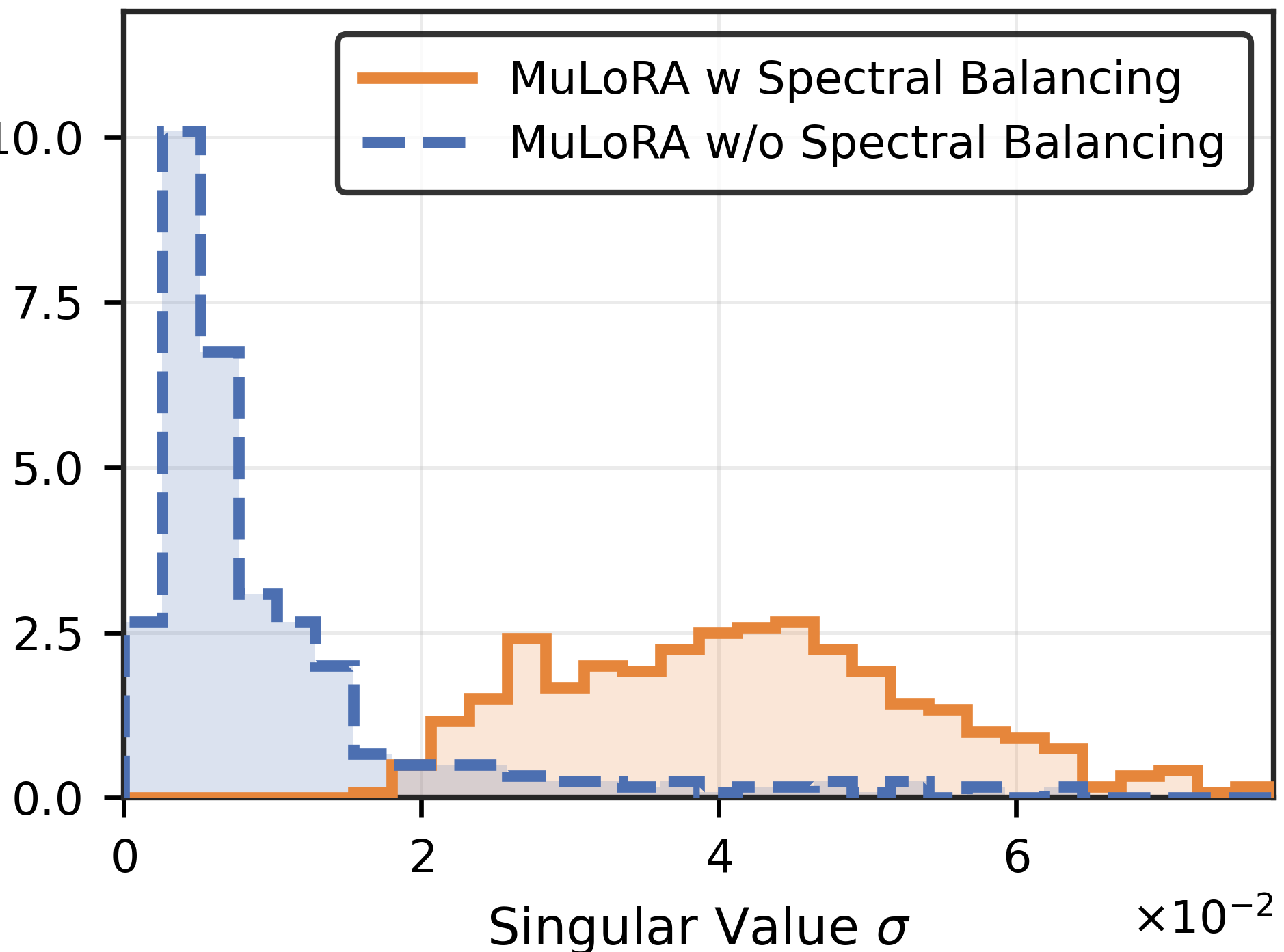}
		\caption{Task 1}
		\label{fig:sp_task1}
	\end{subfigure}
	\hfill
	\begin{subfigure}[t]{0.24\linewidth}
		\centering
		\includegraphics[width=\linewidth]
		{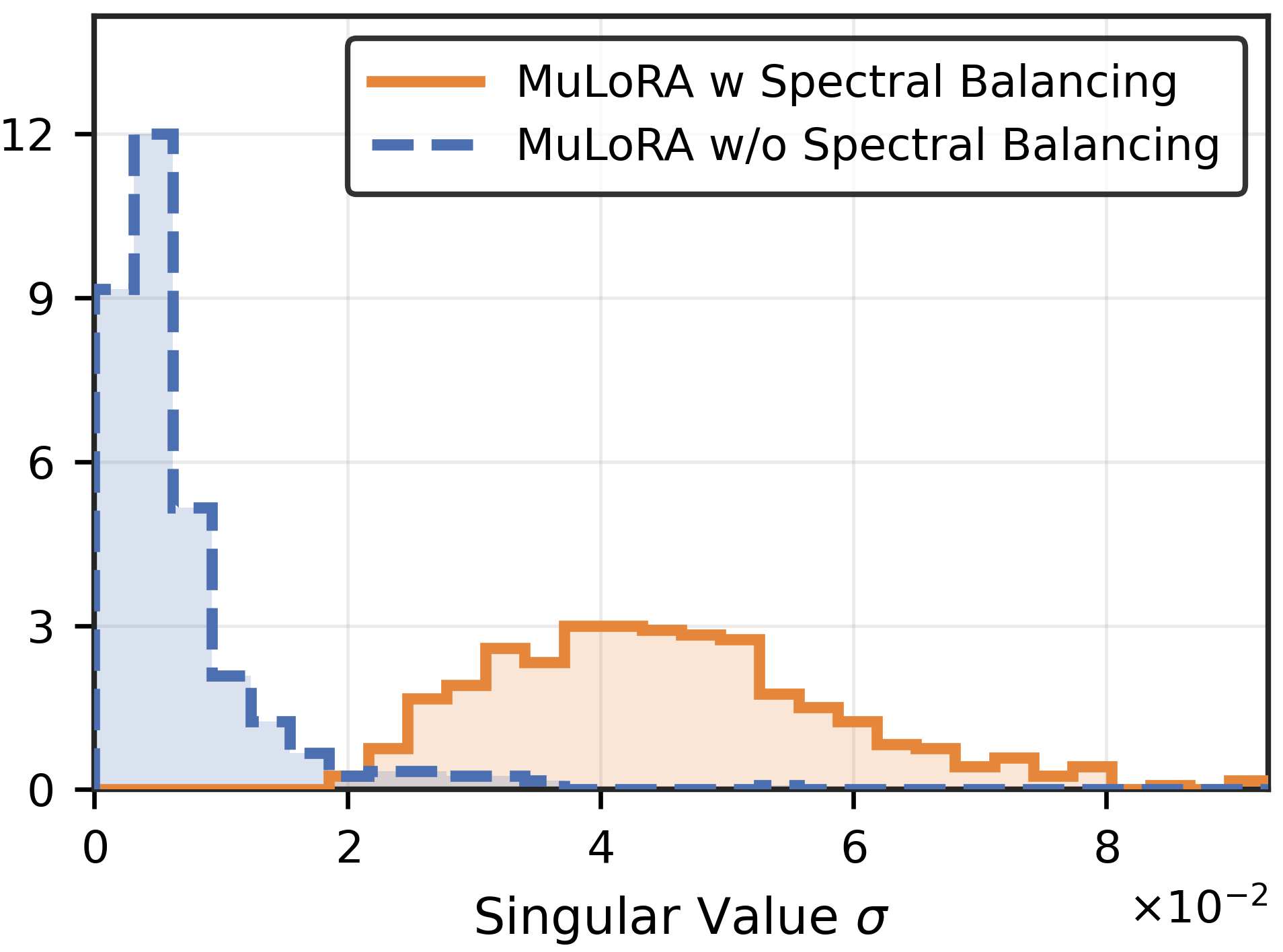}
		\caption{Task 20}
		\label{fig:sp_task20}
	\end{subfigure}
	\hfill
	\begin{subfigure}[t]{0.48\linewidth}
		\centering
		\includegraphics[width=\linewidth]
		{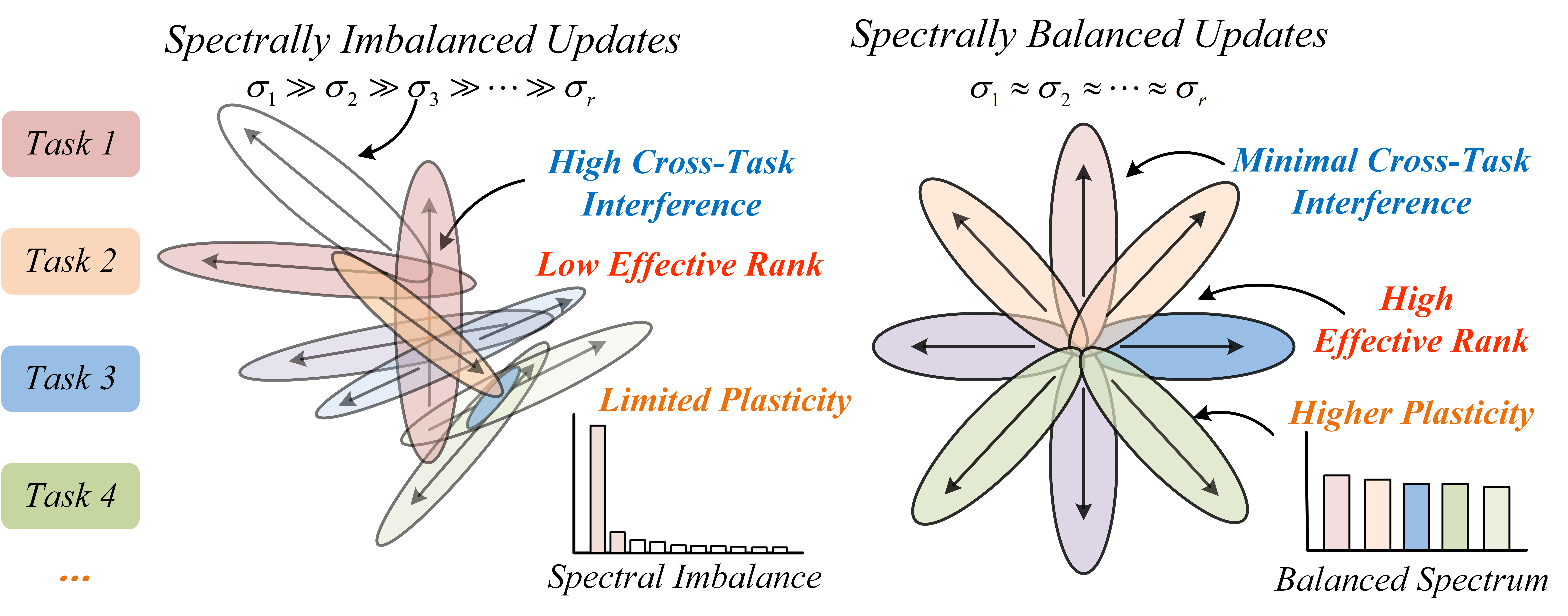}
		\caption{Spectral Capacity Utilization}
		\label{fig:sp_capacity}
	\end{subfigure}
	\vspace{-1.0mm}
	\caption{
		\textbf{Spectral concentration limits effective
			low-rank capacity.}
		(a,b) Singular-value distributions (Frequency vs.\ $\sigma$) of learned
		LoRA updates $\Delta W$ at Tasks 1 and 20, with and without
		spectral balancing.
		More balanced spectra distribute update energy
		across a larger number of modes.
		(c) Conceptual illustration of concentrated
		versus balanced spectra and their relationship
		to effective rank.
		A fixed nominal rank can correspond to markedly
		different levels of spectral utilization.
	}
	\vspace{-4.0mm}
	\label{fig:sp_curve}
\end{figure*}

We propose \texttt{MLoRA},
which connects Relative-Response Spectral Allocation (RRS)
with Subspace Spectral Balancing (SSB).
For each task, \texttt{MLoRA} constructs a
row-orthonormal input basis $A_t$, keeps it fixed,
and applies Muon-style updates to the trainable
up-projection $B_t$.
The central design principle is to select
task-responsive directions using historical information
and preserve the update spectrum through the
LoRA parameterization.

For subspace selection, we maximize current-task
feature response under a regularized historical-response
budget.
Historical whitening converts this objective into
a standard spectral selection problem.
The selected directions are mapped back to the input
space and orthonormalized to construct $A_t$.
Historical information thus provides a graded selection
metric: previously active directions remain eligible
when their relative current-task response is strong.
With exact spectral selection, the resulting subspace
maximizes the worst-direction relative response among
all subspaces of the same dimension.
This provides a subspace-wide characterization of
the directions available for adaptation.

For spectral balancing, we apply approximate polar
orthogonalization to momentum updates of $B_t$.
The ideal transformation equalizes all nonzero singular
values, maximizing spectral participation on the
active update support.
Since $A_tA_t^\top=I_r$, the resulting factor step
and its induced weight step have identical nonzero
singular values and effective rank.
This connects the two components:
RRS determines where adaptation occurs and supplies
an orthonormal basis, while SSB controls how energy
is distributed in each weight step.

Our analysis distinguishes this stepwise control
from the behavior of accumulated updates.
Balanced steps need not produce a balanced sum,
because successive steps can reinforce or cancel
different directions.
We derive cumulative condition-number and spectral
participation bounds under controlled cross-step
anisotropy, identifying conditions under which
stepwise balancing extends to accumulated adaptation.
Empirically, we examine singular-value distributions
and participation-based effective rank alongside
continual-learning accuracy.
Across five benchmarks and eight incremental settings,
\texttt{MLoRA} achieves the highest mean accuracy
in 15 of 16 metrics.
Ablations support jointly selecting task-responsive
subspaces and balancing updates within them.

\textbf{Our main contributions} can be summarized as follows:\\
\textbf{$\bullet$}
We identify \emph{spectral plasticity collapse} in continual
LoRA adaptation: update energy concentrates in a few
singular modes, leaving the effective rank well below
the nominal adapter rank.
We further highlight a challenge in applying Muon:
balancing the two factor steps separately does not
guarantee a spectral balanced weight step.
\\
\textbf{$\bullet$}
We propose \texttt{MLoRA}, a history-aware low-rank
parameterization that transfers Muon's spectral balancing
directly to weight space.
It fixes an orthonormal input basis selected by
current-task response relative to historical exposure
and applies Muon updates to the trainable factor.
This design preserves the nonzero singular values
of each factor step in the induced weight step,
enabling spectral control within a task-responsive
adaptation space.
\\
\textbf{$\bullet$}
We show that exact subspace selection maximizes the
worst-direction relative response and derive cumulative
spectral bounds under controlled cross-step anisotropy.
Across five benchmarks and eight incremental settings,
\texttt{MLoRA} achieves the highest mean accuracy
in 15 of 16 reported metrics, with ablations supporting
the complementary benefits of RRS and SSB.

\section{Related Works}

\textbf{PEFT-based Continual Learning.}
Parameter-efficient fine-tuning (PEFT) enables continual adaptation of pretrained models through lightweight trainable modules.
Prompt-based methods~\cite{codaCVPR23, cpromptCVPR24} learn and select prompts to accommodate sequential tasks, while PGP~\cite{pgpICLR24} and VPT-NSP~\cite{vptnspNIPS-24} derive projection constraints for mitigating interference.
Adapter-based methods~\cite{DIA-cvpr25, ssaitCVPR24, he2025ckaa, wang2026ekpc} preserve and integrate task knowledge through lightweight modules, representation alignment, or feature compensation.
LoRA-based approaches further structure adaptation within low-rank parameter spaces.
SD-LoRA~\cite{sdloraICLR25} decouples direction and magnitude learning; O-LoRA~\cite{olora-acl2023} and InfLoRA~\cite{infloraCVPR24} restrict overlap with historical subspaces to reduce interference.
SplitLoRA~\cite{qiu2025splitlora}, BiLoRA~\cite{biloraCVPR25}, and PLAN~\cite{planICCV25} further explore subspace decomposition, orthogonality, and allocation to balance stability and plasticity.
Complementing these approaches, EWC-LoRA~\cite{zheng2026revisiting} regularizes a shared low-rank update using parameter importance estimated in the full weight space.
Closely related to our work, EBLoRA~\cite{gu2026spectral} links
spectral imbalance to forgetting and enforces equal nonzero
singular values through orthonormal factors and a shared scalar
magnitude. \texttt{MLoRA} addresses the stability--plasticity
trade-off by selecting subspaces with strong current-task response
relative to historical exposure and balancing momentum steps
within them.

\begin{figure*}[!t]
	\centering
	\includegraphics[width=1.00\textwidth]{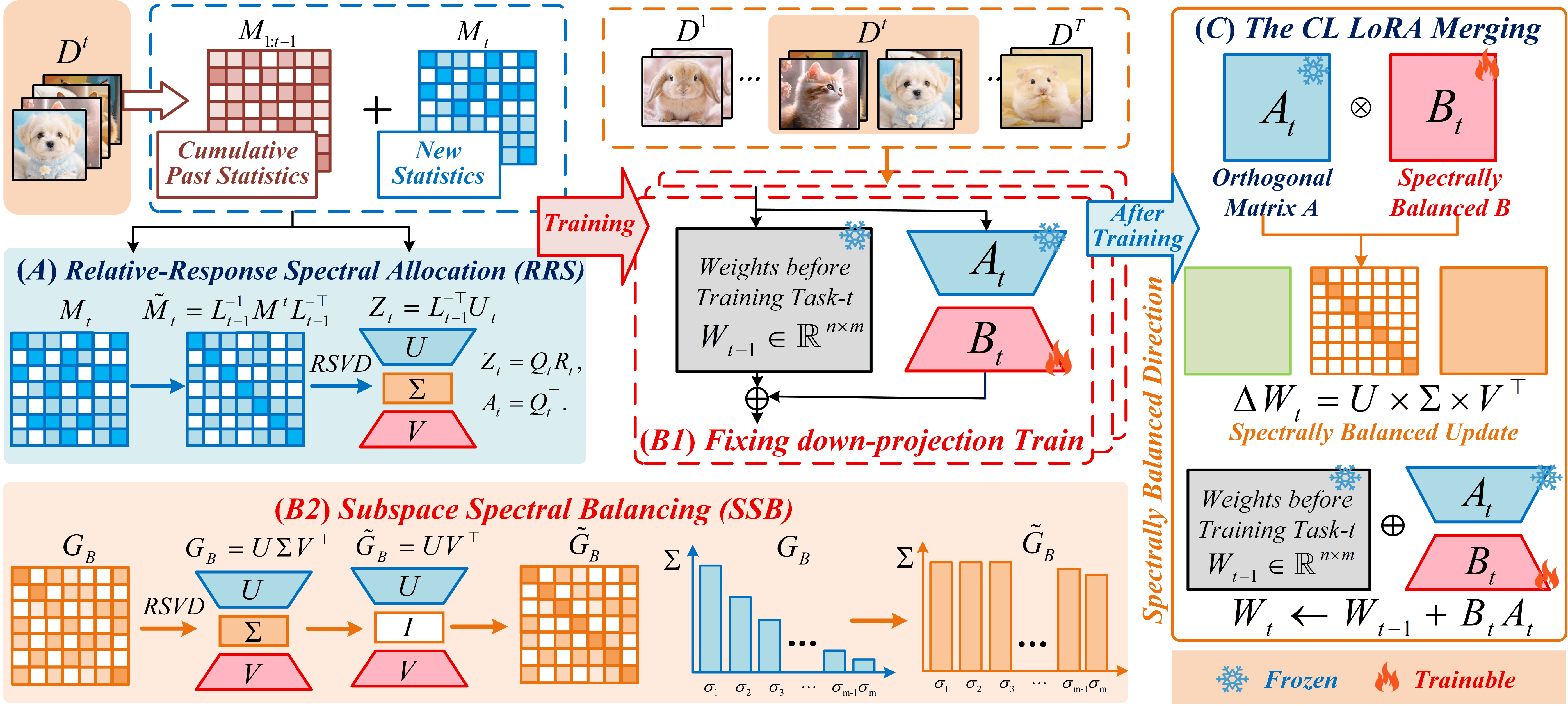}
	\vspace{-3mm}
\caption{
	\textbf{Relative-response subspace selection and spectrally controlled optimization.}
	(A) Historical whitening identifies directions with strong current-task
	response relative to accumulated historical response, from which
	we construct a row-orthonormal basis $A_t$.
	(B1) We freeze the backbone and $A_t$, and train only $B_t$.
	(B2) Approximate polar orthogonalization reshapes the momentum spectrum,
	balancing individual optimization steps within the selected subspace.
	(C) After each task, the learned update is merged into the backbone as
	$W_t=W_{t-1}+B_tA_t$.
}\label{fig: overall_framework}
	\vspace{-1mm}
\end{figure*}

\section{Methodology}

\textbf{Problem Definition.}
We study class-incremental image classification over $T$ tasks.
At stage $t$, the training set $\mathcal{D}^t=\{(x_i^t,y_i^t)\}_{i=1}^{N^t}$
introduces a set of classes $\mathcal{C}^t$ that have not appeared in earlier stages.
The model adapts to $\mathcal{D}^t$ without replaying past examples or their
individual features; aggregate feature statistics may be retained.
After each stage $t$, the model predicts over $\bigcup_{s=1}^{t}\mathcal{C}^s$
without knowing the task identity of a test image.

\textbf{Overall framework.}
Figure~\ref{fig: overall_framework} illustrates the proposed framework.
Following existing PEFT-based methods~\cite{l2pCVPR22, ssaitCVPR24},
we adopt a pretrained Vision Transformer (ViT)~\cite{ViT} as the backbone.
For task $t$, each adapted weight matrix
$W_{t-1}\in\mathbb{R}^{n\times m}$ remains frozen and is augmented
with a rank-$r$ update $B_tA_t$, where
$A_t\in\mathbb{R}^{r\times m}$ and $B_t\in\mathbb{R}^{n\times r}$.
Historical whitening identifies a task-adaptive subspace with strong
current-task response relative to accumulated historical response.
We construct a row-orthonormal basis $A_t$ of this subspace,
satisfying $A_tA_t^\top=I_r$, and keep it fixed throughout the task.
Starting from $B_t=0$, we optimize $B_t$ using Muon-style approximate
polar orthogonalization of momentum updates.
The orthonormal basis preserves the nonzero singular values
between each factor update and its induced weight update,
connecting subspace selection with stepwise spectral control.
For an input feature $x\in\mathbb{R}^{m}$, the adapted layer computes
$y=(W_{t-1}+B_tA_t)x\in\mathbb{R}^{n}$.
After training, we merge the learned update into the backbone as
$W_t=W_{t-1}+B_tA_t$ and construct a new basis for the next task.

\subsection{Relative-Response Spectral Allocation}
\vspace{-2mm}
\label{sec:spectral_init}

\textbf{Motivation: shared feature directions also can support new learning.}
Isolation-based methods such as InfLoRA restrict adaptation
to directions orthogonal to a retained historical feature
subspace to mitigate interference.
However, historical and current tasks may rely on shared
directions: excluding them can discard useful adaptation
capacity, while low historical response alone does not
ensure relevance to the incoming task.
We therefore treat historical response as a graded metric
for subspace selection.
A previously active direction remains eligible when its
current-task response is sufficiently strong relative to
its historical exposure.
Our guiding principle is:
\emph{select directions by their current-to-history response,
	then balance updates within the selected subspace.}

\textbf{Relative-Response Spectral Allocation (RRS).}
Let $X_t\in\mathbb{R}^{(n_t n_b)\times m}$ collect the input
features of task $t$, where $n_t$ is the number of samples
and $n_b$ is the number of patch tokens per sample.
Let $M_t=(X_t)^\top X_t$ denote the current-task
second-moment statistics and
$M_{1:t-1}=\sum_{i=1}^{t-1}M_i$ the accumulated
historical statistics.
We define the regularized historical metric as
$H_{t-1}=M_{1:t-1}+\epsilon I$, where $\epsilon>0$.
For a direction $v\neq0$, we measure its relative response by
$
	\mathcal{R}_t(v)
	=
	\frac{v^\top M_t v}{v^\top H_{t-1} v}
	=
	\frac{\|X_t v\|_2^2}
	{\sum_{i<t}\|X_i v\|_2^2+\epsilon\|v\|_2^2}.
	\label{eq:relative_response}
$
As illustrated in Figure~\ref{fig:relative_response},
this criterion favors directions with strong current-task response
relative to their historical response. Directions shared with
previous tasks remain available when they are sufficiently
responsive to the current task.
To select a rank-$r$ subspace, we solve
\vspace{-0mm}
\begin{equation}
	\max_{Z\in\mathbb{R}^{m\times r}}
	\operatorname{tr}(Z^\top M_t Z)
	\quad
	\text{s.t.}\quad Z^\top H_{t-1} Z=I_r.
	\label{eq:relative_subspace}
\end{equation}
The constraint gives each basis direction the same
regularized historical-response budget and makes the
directions orthogonal under the historical metric.
Within this budget, the objective selects a subspace
with the largest total current-task response.
This provides a feature-based criterion for allocating
adaptation capacity, without directly estimating
task loss or forgetting.

\begin{figure}[t]
	\centering
	\vspace{-2mm}
	\includegraphics[width=\linewidth]{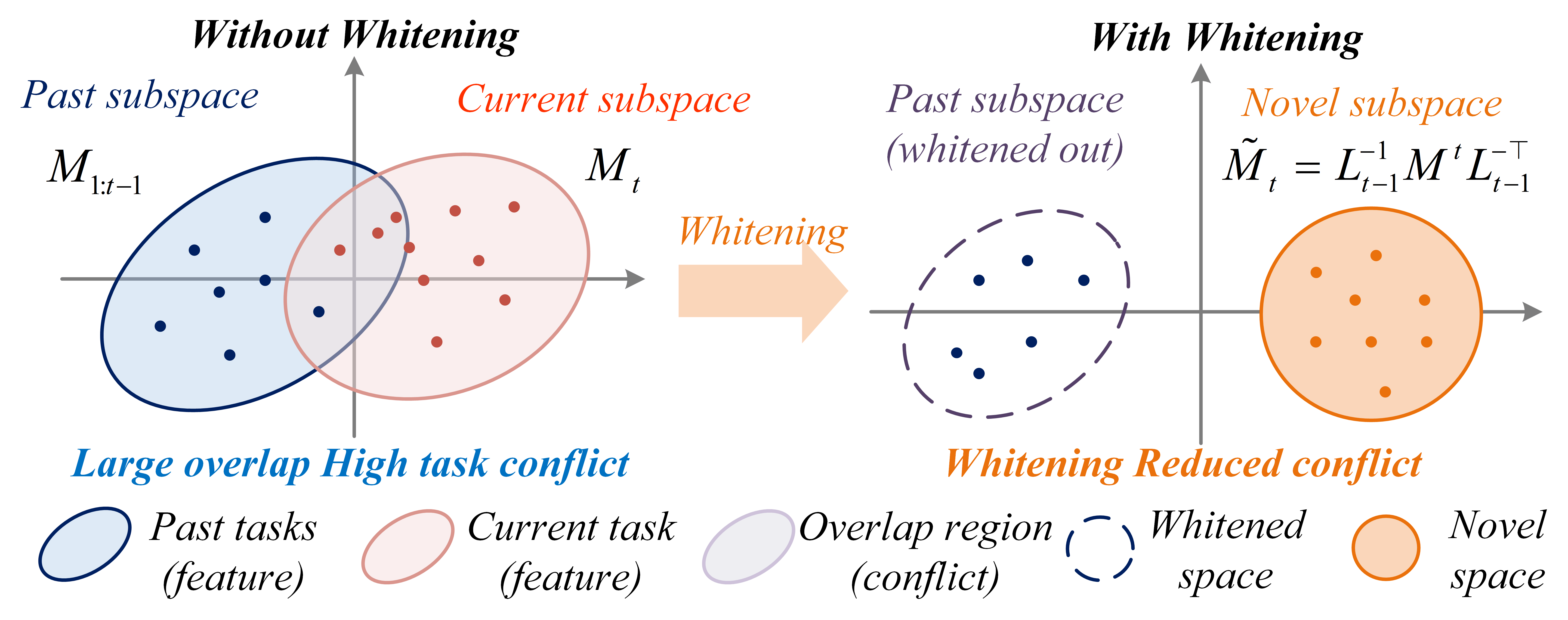}
	\vspace{-2mm}
	\caption{Allocating LoRA capacity using current-task
		response relative to historical feature statistics.}
	\label{fig:relative_response}
	\vspace{-6mm}
\end{figure}

\textbf{Solving the subspace selection problem with whitening.}
We solve Eq.~(\ref{eq:relative_subspace}) by transforming the
historical-metric constraint into a standard orthogonality constraint.
Let
$
H_{t-1}=L_{t-1}L_{t-1}^{\top}
$
be the Cholesky factorization of $H_{t-1}\succ 0$.
With the change of variables
$
U=L_{t-1}^{\top}Z
$,
or equivalently
$
Z=L_{t-1}^{-\top}U
$,
Eq.~(\ref{eq:relative_subspace}) becomes
\begin{equation*}
	\max_{U^\top U=I_r}
	\operatorname{tr}
	\left(
	U^\top \widetilde{M}_t U
	\right),
	\widetilde{M}_t
	=
	L_{t-1}^{-1}
	M^t
	L_{t-1}^{-\top}.
	\label{eq:whitened_objective}
\end{equation*}
Hence, the optimal whitened subspace is spanned by the leading
$r$ eigenvectors of $\widetilde{M}_t$, which correspond to directions
with the strongest current-task response relative to historical
response.

In practice, we approximate this dominant eigenspace using randomized
SVD~\cite{halko2011finding}:
$
\widetilde U_t,\Sigma_t,\widetilde V_t
=
\texttt{RSVD}
\bigl(
\widetilde{M}_t,r
\bigr),
$
where $\widetilde U_t\in\mathbb{R}^{m\times r}$ contains orthonormal
approximate leading spectral directions in the whitened coordinates.
We then map them back to the original feature space as
$
	Z_t
	=
	L_{t-1}^{-\top}\widetilde U_t.
	\label{eq:map_back_subspace}
$
Because the mapped basis $Z_t$ is generally not Euclidean-orthonormal,
we perform a thin QR factorization
$
	Z_t = Q_tR_t,
	Q_t^\top Q_t=I_r,
	\label{eq:rrs_qr}
$
and define the fixed LoRA down-projection as
$
	A_t=Q_t^\top.
	\label{eq:rrs_basis}
$
Since $R_t$ is nonsingular whenever $Z_t$ has full column rank,
$
\operatorname{range}(Q_t)
=
\operatorname{range}(Z_t),
$
so the QR step changes only the basis representation and preserves the
selected relative-response subspace exactly. Meanwhile,
$
A_tA_t^\top=I_r
$,
which ensures that every factor update $\delta B_t$ and its induced
weight update $\delta B_tA_t$ have identical nonzero singular values.
Thus, RRS preserves the generalized-eigen subspace selected by the
relative-response criterion while providing the Euclidean-orthonormal
parameterization required for exact spectral transfer in the subsequent
SSB optimization.
For the first task, no historical statistics are available and
$H_{0}=\epsilon I$. Therefore,
$\widetilde M_1=\epsilon^{-1}M_1$, and RRS reduces to selecting the
dominant current-task feature subspace.

\subsection{Subspace Spectral Balancing}
\label{sec:spectral_training}
\textbf{Motivation: spectral collapse limits capacity utilization.}
Even within a task-responsive subspace, adaptation updates
can concentrate most of their energy in a few singular modes,
leaving the remaining directions weakly utilized.
We refer to this concentration as \emph{spectral collapse}:
the effective rank of the update can be much lower than
the nominal LoRA rank, even without a loss of algebraic rank.
Subspace selection alone therefore does not ensure that
the available capacity is effectively used.
To address this issue, we apply Muon-style spectral balancing
to the optimization step, reducing the dominance of a few
modes within the selected subspace.

\textbf{Connecting factor steps to weight steps.}
For task $t$, we parameterize the adapted weight as
$
W_{t,k}=W_{t-1}+B_{t,k}A_t,
A_tA_t^\top=I_r,
\label{eq:fixed_basis_adaptation}
$
where $k$ indexes optimization steps.
We freeze $A_t$ and train $B_{t,k}$ from $B_{t,0}=0$.
Unlike updating both factors, this yields a weight step
$S_k=\delta B_kA_t$, where
$\delta B_k=B_{t,k+1}-B_{t,k}$, satisfying
$
S_kS_k^\top
=
\delta B_kA_tA_t^\top\delta B_k^\top
=
\delta B_k\delta B_k^\top.
\label{eq:step_spectrum_transfer}
$
Thus, the factor step and the weight step have identical
nonzero singular values.
Controlling the spectrum of $\delta B_k$ therefore
controls the spectrum of the actual weight step,
without distortion from $A_t$.

\textbf{Subspace Spectral Balancing (SSB).}
Let $G_k\in\mathbb{R}^{n\times r}$ denote the momentum
matrix used to update $B_{t,k}$.
For $G_k\neq0$, write its compact SVD as
$
	G_k=U_k\Sigma_kV_k^\top,
	\Sigma_k=\operatorname{diag}(\sigma_1,\ldots,\sigma_q),
	\label{eq:update_svd}
$
where $q=\operatorname{rank}(G_k)$ and $\sigma_j>0$.
The ideal polar transformation preserves the active
singular directions and sets their singular values
to one:
\begin{equation}
	D_k=\mathcal{O}(G_k)=U_kV_k^\top,\:
	B_{t,k+1}=B_{t,k}-\eta_k\mathcal{O}(G_k).
	\label{eq:polar_update}
\end{equation}
This distributes the step energy equally across
all active modes in Figure \ref{fig:spectral-training}.
SSB thus maximizes spectral participation on the
existing update support; it does not create additional
modes when $G_k$ is rank deficient.
Together, RRS and SSB determine which input directions
are available for adaptation and how evenly each step
uses its active singular modes.

\textbf{Efficient approximate implementation.}
Computing an SVD at every step is unnecessary.
Following Muon~\cite{jordan2024muon,bernstein2025derivingmuon},
we approximate the polar transformation using
Newton--Schulz polynomial iterations.
Starting from $Y_0=G_k/(\|G_k\|_F+\delta)$, with
$\delta>0$, we compute
\begin{equation}
	K_j=Y_j^\top Y_j,
	\:
	Y_{j+1}=aY_j+bY_jK_j+cY_jK_j^2,
	\label{eq:newton_schulz}
\end{equation}
for $j=0,\ldots,J-1$, and update
$B_{t,k+1}=B_{t,k}-\eta_kY_J$.
Each iteration reshapes the singular values through
a low-degree polynomial.
Since $K_j\in\mathbb{R}^{r\times r}$, the computation
exploits the small LoRA dimension.
Finite iterations provide approximate spectral balancing,
while spectral transfer from the factor step to the induced weight step remains exact.
The same spectrum-preserving identity holds between
$B_{t,k}$ and the accumulated adapter $B_{t,k}A_t$.
However, balanced individual steps need not yield
a balanced sum, because successive steps can reinforce
or cancel different directions.
Our theoretical analysis characterizes cumulative
spectral behavior under controlled cross-step anisotropy.

\begin{figure*}[t]
	\centering
	\includegraphics[width=\linewidth]{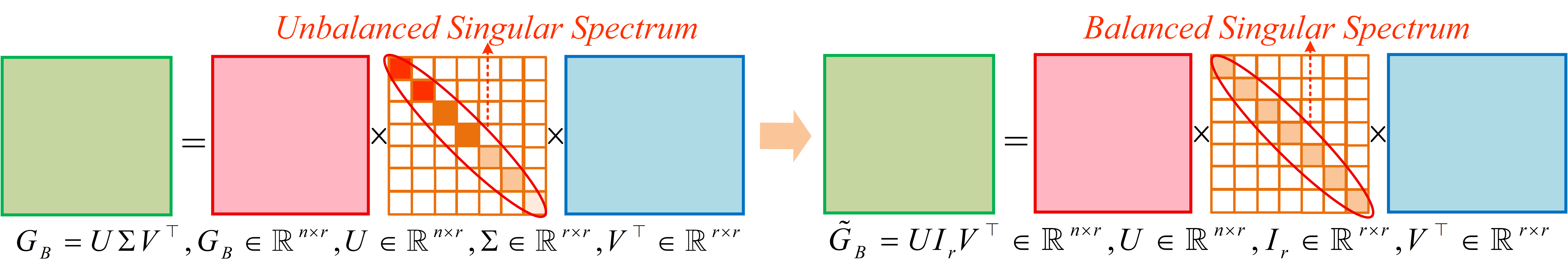}
	\vspace{-5mm}
	\caption{Spectral control of LoRA updates.
		An orthonormal input basis transfers the singular
		spectrum of each factor update to the induced
		weight update without distortion.}
	\label{fig:spectral-training}
	\vspace{-3mm}
\end{figure*}

\section{Theoretical Analysis}
\label{sec:theoretical_analysis}

\paragraph{Subspace-wide relative-response optimality.}
RRS first identifies a dominant subspace in the historically
whitened coordinates and then maps it back to the original feature
space. The following result shows that this mapped subspace is
max--min optimal for the original relative-response criterion, while
Euclidean re-orthogonalization preserves the guarantee.

\begin{proposition}[Subspace-Wide Relative-Response Optimality]
	\label{prop:relative_response}
	Let
	$
	H_{t-1}=L_{t-1}L_{t-1}^{\top}\succ0
	$
	and
	$
	M_t\succeq0
	$,
	and define
	$
	\widetilde M_t
	=
	L_{t-1}^{-1}
	M^t
	L_{t-1}^{-\top}.
	$
	Let
	$
	\lambda_1\ge\cdots\ge\lambda_m\ge0
	$
	be the eigenvalues of $\widetilde M_t$, and let
	$
	U_r\in\mathbb{R}^{m\times r}
	$
	contain its leading $r$ orthonormal eigenvectors.
	Define the corresponding subspace in the original feature space as
	$
	\mathcal S_t
	=
	\operatorname{range}
	\!\left(
	L_{t-1}^{-\top}U_r
	\right).
	$
	Then every nonzero $v\in\mathcal S_t$ satisfies
	$
		\lambda_r
		\le
		\frac{v^\top M^t v}
		{v^\top H_{t-1}v}
		\le
		\lambda_1.
		\label{eq:relative_subspace_response}
	$
	Moreover,
	$
		\max_{\dim(\mathcal S)=r}
		\;
		\min_{v\in\mathcal S\setminus\{0\}}
		\frac{v^\top M^t v}
		{v^\top H_{t-1}v}
		=
		\lambda_r,
		\label{eq:relative_maxmin_response}
	$
	and the maximum is attained by $\mathcal S_t$.
	If
	$
	L_{t-1}^{-\top}U_r=Q_tR_t
	$
	is a thin QR factorization and
	$
	A_t=Q_t^\top
	$,
	then
	$
	A_tA_t^\top=I_r,
	\operatorname{row}(A_t)=\mathcal S_t.
	$
	Hence, Euclidean re-orthogonalization preserves the
	max--min-optimal RRS subspace while providing the
	row-orthonormal basis required for exact spectral transfer.
\end{proposition}

Proposition~\ref{prop:relative_response} provides a uniform
guarantee over the entire LoRA adaptation space:
even its least responsive direction achieves relative response
at least $\lambda_r$, and no other $r$-dimensional subspace can
attain a larger worst-direction response.
The result follows from the change of variables
$
u=L_{t-1}^{\top}v
$,
under which
$
\frac{v^\top M_t v}
{v^\top H_{t-1}v}
=
\frac{u^\top\widetilde M_tu}
{u^\top u},
$
together with the Courant--Fischer theorem.
Importantly, the subsequent QR factorization changes only the
basis of the selected subspace, not the subspace itself.
Thus, RRS simultaneously preserves its relative-response
optimality and yields $A_tA_t^\top=I_r$, directly connecting
subspace allocation to the spectrum-preserving LoRA
parameterization used by SSB.

\begin{figure*}[t]
	\centering
	\captionsetup[subfigure]{font=small,labelfont=bf}
	\begin{subfigure}[t]{0.245\linewidth}
		\centering
		\includegraphics[width=\linewidth]{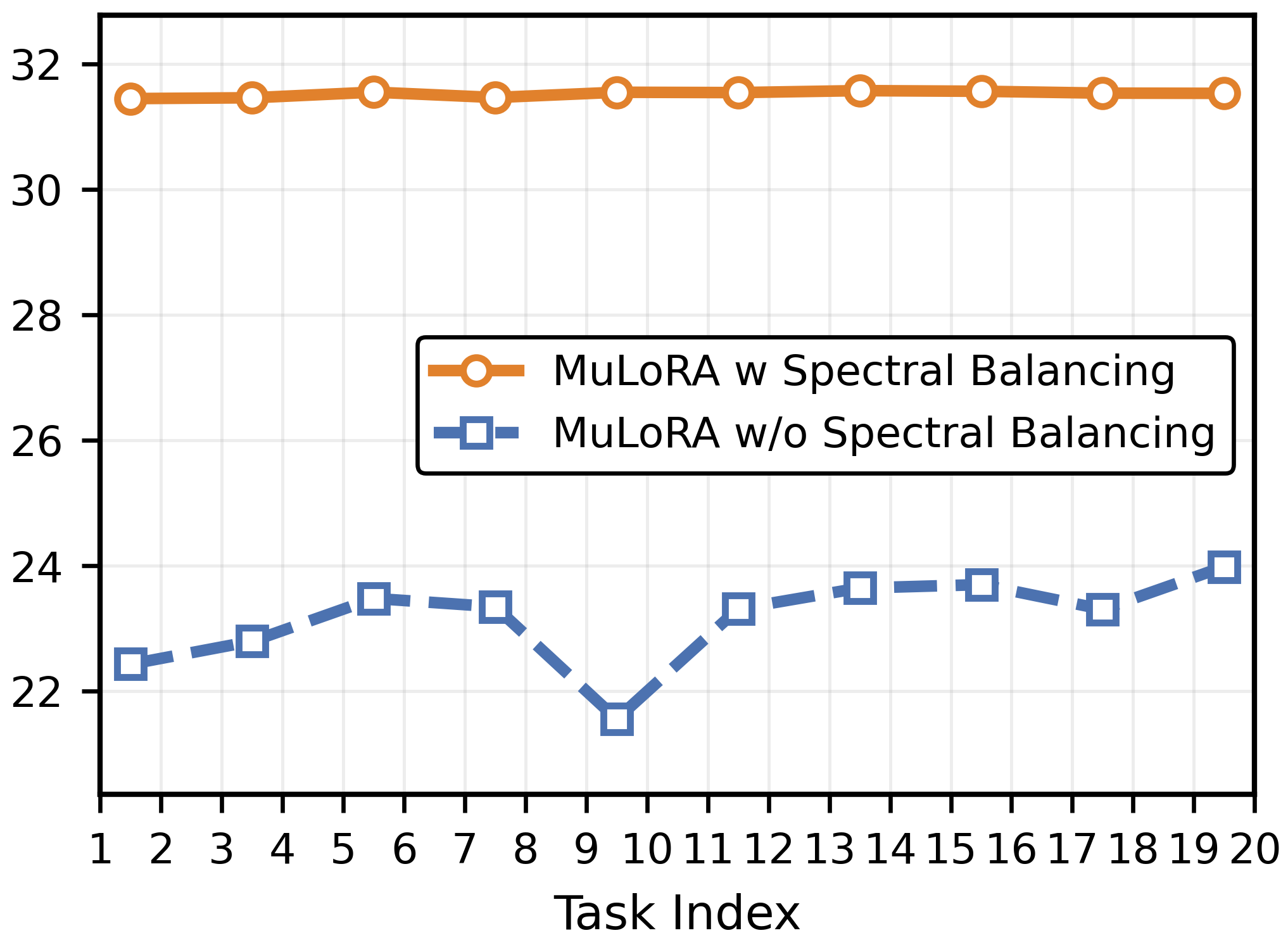}
		\caption{Effective rank}
	\end{subfigure}
	\hfill
	\begin{subfigure}[t]{0.245\linewidth}
		\centering
		\includegraphics[width=\linewidth]{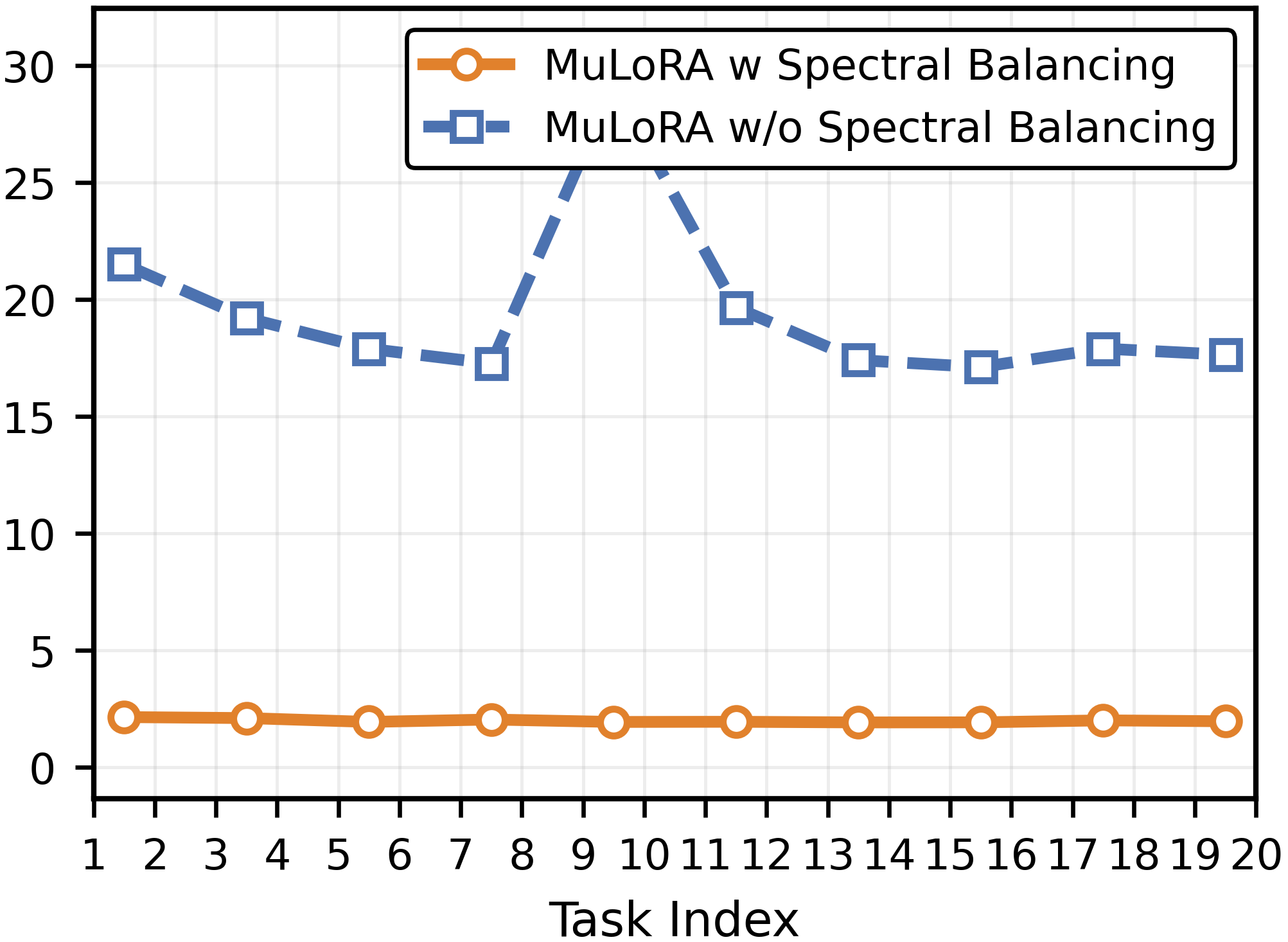}
		\caption{$\kappa_+$}
	\end{subfigure}
	\hfill
	\begin{subfigure}[t]{0.245\linewidth}
		\centering
		\includegraphics[width=\linewidth]{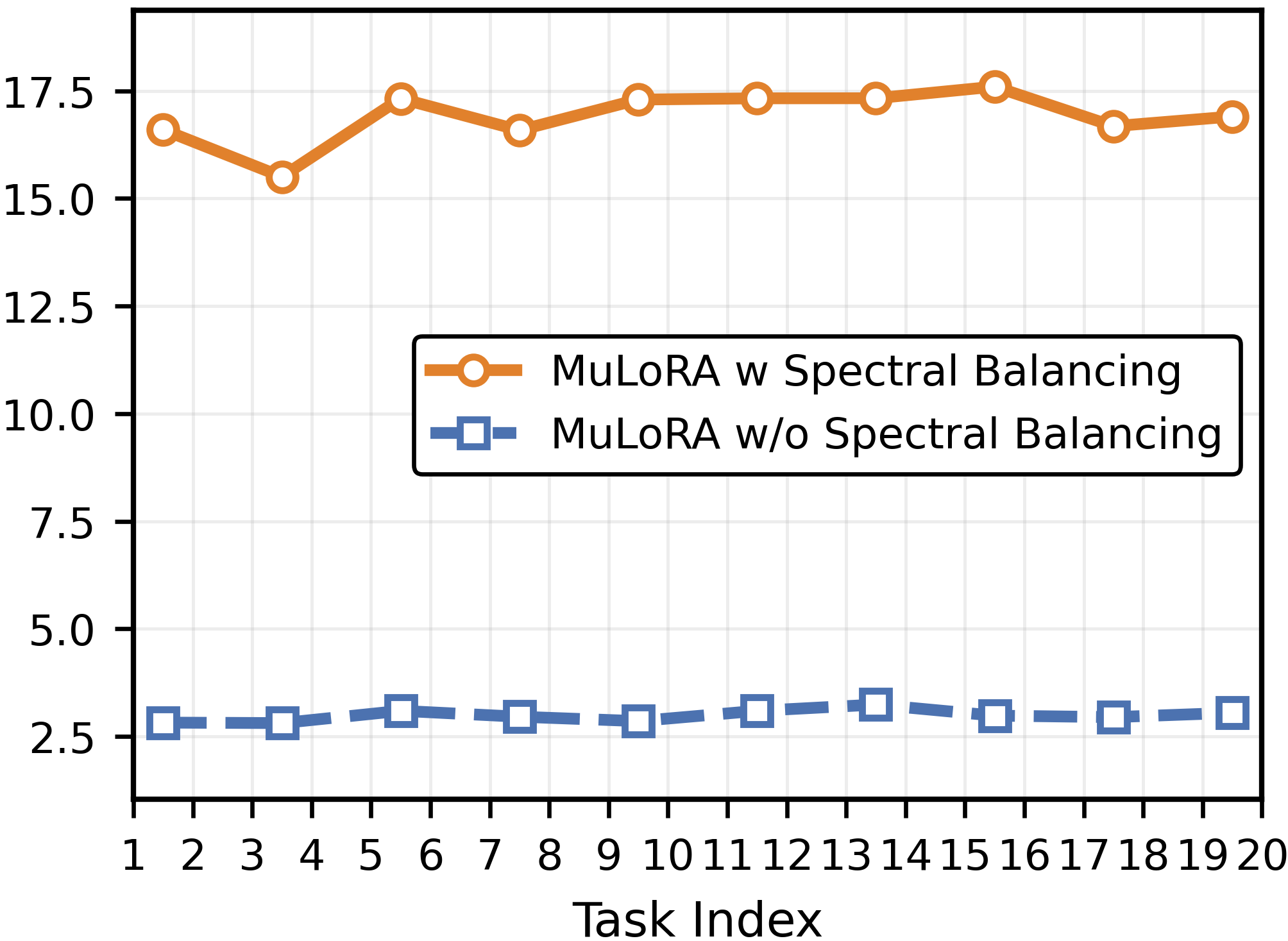}
		\caption{Stable rank}
	\end{subfigure}
	\hfill
	\begin{subfigure}[t]{0.245\linewidth}
		\centering
		\includegraphics[width=\linewidth]{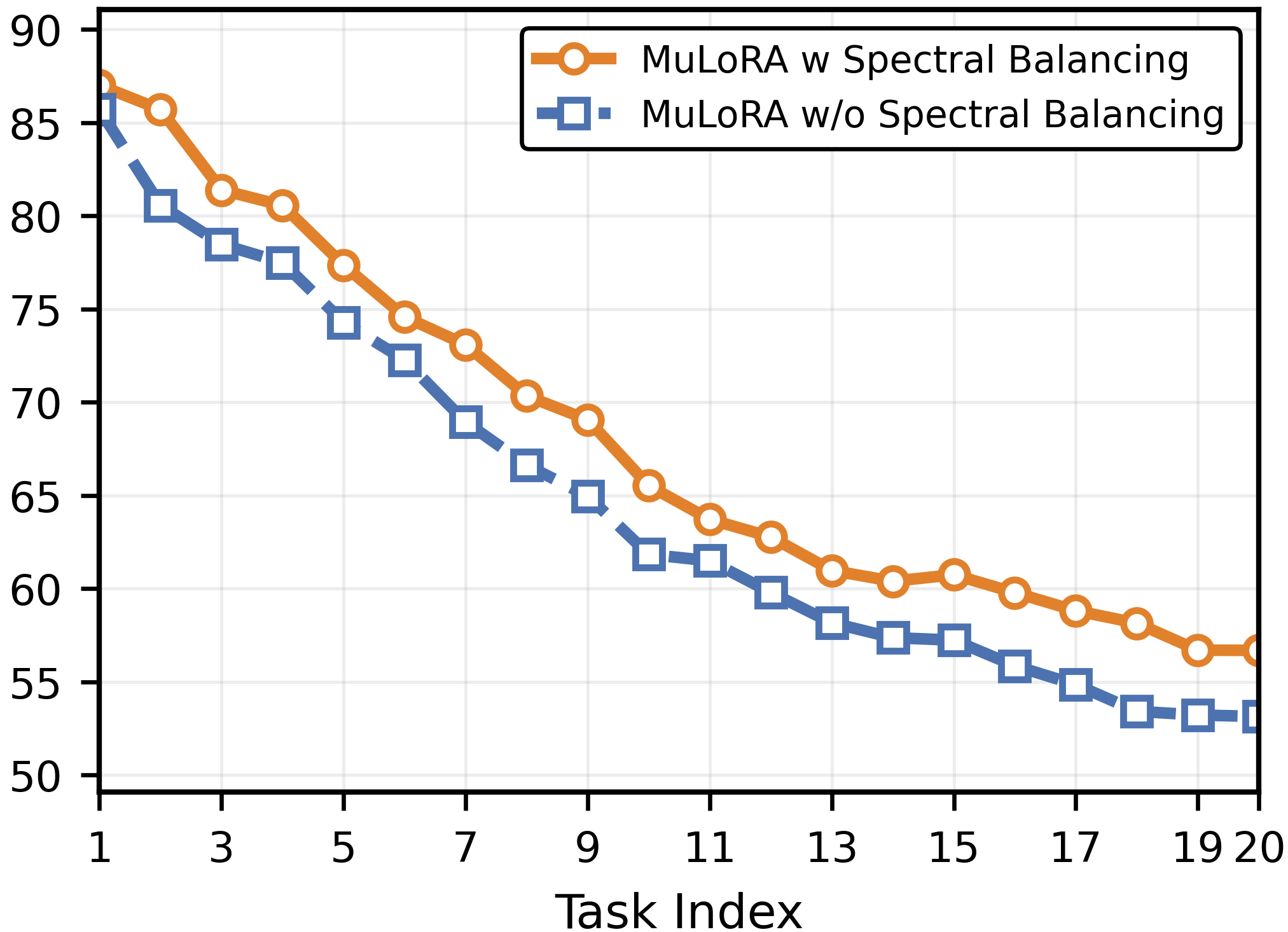}
		\caption{Overall Test Acc}
	\end{subfigure}
	\vspace{-1.0mm}
	\caption{\textbf{Spectral balancing prevents rank collapse.}
		We track effective rank, $\kappa_{+}$, stable rank of $\Delta W$, and overall test accuracy
		across tasks, comparing \texttt{MLoRA} with and without spectral balancing. SSB maintains high effective rank and stable rank while keeping the nonzero spectrum well conditioned throughout sequential adaptation. Importantly, this geometric advantage persists over the same tasks for which overall class-incremental accuracy is better preserved, suggesting that spectral utilization is closely associated with retained adaptation capacity.}
	\vspace{-4.0mm}
	\label{fig:rank}
\end{figure*}
\subsection{Cumulative spectral balance}

Stepwise orthogonalization balances individual updates, but
their accumulation also depends on cross-step interactions.
We characterize how cumulative anisotropy controls the
condition number, stable rank, and entropy-based effective
rank of the accumulated update.
For a nonzero matrix $B$ with positive singular values
$\sigma_1(B),\ldots,\sigma_q(B)$, define
$
\operatorname{sr}(B)
=
\frac{\|B\|_F^2}{\|B\|_2^2},
\operatorname{erank}(B)
=
\exp\!\left(-\sum_{j=1}^{q}p_j\log p_j\right),
p_j
=
\frac{\sigma_j(B)}{\sum_{i=1}^{q}\sigma_i(B)}.
$
Here $\log$ denotes the natural logarithm.
For a full-column-rank matrix
$B\in\mathbb{R}^{n\times r}$, let
$\kappa(B)=\sigma_1(B)/\sigma_r(B)$.
For a general nonzero matrix, $\kappa_+$ denotes the ratio
of its largest to smallest positive singular value in Figure \ref{fig:rank}.

\begin{theorem}[Cumulative Spectral Balance]
	\label{thm:cumulative_balance}
	Let $K\ge1$, $n\ge r\ge1$, $B_0=0$, and
	$
	B_K=-\sum_{k=1}^{K}\eta_kD_k,
	\eta_k>0,
	D_k\in\mathbb{R}^{n\times r},
	D_k^\top D_k=I_r.
	$
	Define
	$
	C_K
	=
	\sum_{1\le i<j\le K}\eta_i\eta_j
	\left(D_i^\top D_j+D_j^\top D_i\right),
	$
$
	a_K
	=
	\sum_{k=1}^{K}\eta_k^2
	+
	\frac{\operatorname{tr}(C_K)}{r},
	E_K
	=
	C_K-\frac{\operatorname{tr}(C_K)}{r}I_r.
$
	Suppose $a_K>0$ and
	$\|E_K\|_2\le\rho a_K$ for some $\rho\in[0,1)$.
	Then $B_K$ has full column rank, and
	$
		\sqrt{a_K(1-\rho)}
		\le \sigma_j(B_K)
		\le \sqrt{a_K(1+\rho)},
		 j=1,\ldots,r.
		\label{eq:cumulative_spectrum_bound}
	$
	Moreover,
	$
		\kappa(B_K)
		=
		\sqrt{
			\frac{a_K+\lambda_{\max}(E_K)}
			{a_K+\lambda_{\min}(E_K)}
		}
		\le
		\sqrt{\frac{1+\rho}{1-\rho}},
		\label{eq:cumulative_condition_bound}
		\operatorname{sr}(B_K)
		=
		\frac{ra_K}{a_K+\lambda_{\max}(E_K)}
		\ge
		\frac{r}{1+\rho},
		\label{eq:cumulative_stable_rank_bound}
		\operatorname{erank}(B_K)
		\ge
		\frac{r}{1+\|E_K\|_F^2/(ra_K^2)}
		\ge
		\frac{r}{1+\rho^2}.
		\label{eq:cumulative_entropy_rank_bound}
	$
	If $A\in\mathbb{R}^{r\times m}$ satisfies $AA^\top=I_r$,
	then
	$
		\kappa_+(B_KA)=\kappa(B_K),
		\operatorname{sr}(B_KA)=\operatorname{sr}(B_K),
		\operatorname{erank}(B_KA)=\operatorname{erank}(B_K).
	$
\end{theorem}

Theorem~\ref{thm:cumulative_balance} follows from the
decomposition $B_K^\top B_K=a_KI_r+E_K$.
The scalar $a_K=\|B_K\|_F^2/r$ determines the average
squared singular value, whereas $E_K/a_K$ measures
relative spectral anisotropy.
A smaller admissible $\rho$ tightens the condition-number
upper bound and raises the stable-rank and effective-rank
lower bounds.
When $E_K=0$, all singular values are equal, giving
$\kappa(B_K)=1$ and
$\operatorname{sr}(B_K)=\operatorname{erank}(B_K)=r$.
These guarantees are conditional on cumulative anisotropy:
stepwise orthogonalization alone does not ensure that the
condition holds.

\begin{corollary}[Robust Cumulative Spectral Balance]
	\label{cor:approx_cumulative_balance}
	Let $K\ge1$, $n\ge r\ge1$, and
	$
	B_K=-\sum_{k=1}^{K}\eta_kD_k,
	\eta_k>0,
	D_k\in\mathbb{R}^{n\times r},
	$
	where
	$
	\|D_k^\top D_k-I_r\|_2\le\varepsilon_k,
	\varepsilon_k\ge0.
	$
	Define $C_K$, $E_K$, and $a_K$ as in
	Theorem~\ref{thm:cumulative_balance}, and let
	$
	\delta_K
	=
	\sum_{k=1}^{K}\eta_k^2\varepsilon_k.
	$
	Suppose $a_K>0$ and
	$\|E_K\|_2\le\rho a_K$ for some $\rho\in[0,1)$.
	Then
	$
		a_K(1-\rho)-\delta_K
		\le \sigma_j(B_K)^2
		\le a_K(1+\rho)+\delta_K, j=1,\ldots,r.
		\label{eq:robust_cumulative_spectrum}
$
	If $\delta_K<a_K(1-\rho)$, define
	$\tau_K=\delta_K/a_K$.
	Then $B_K$ has full column rank, and
$
		\kappa(B_K)
		\le
		\sqrt{
			\frac{1+\rho+\tau_K}
			{1-\rho-\tau_K}
		},
		\label{eq:robust_condition_bound}
		\operatorname{sr}(B_K)
		\ge
		\frac{r(1-\tau_K)}
		{1+\rho+\tau_K},
		\label{eq:robust_stable_rank_bound}
		\operatorname{erank}(B_K)
	\ge
		\frac{r}{
			1+\left(
			\frac{\rho+\tau_K}{1-\tau_K}
			\right)^2
		}.
		\label{eq:robust_entropy_rank_bound}
$
	If $A\in\mathbb{R}^{r\times m}$ satisfies $AA^\top=I_r$,
	the same bounds hold for
	$\kappa_+(B_KA)$, $\operatorname{sr}(B_KA)$,
	and $\operatorname{erank}(B_KA)$, respectively.
\end{corollary}

Corollary~\ref{cor:approx_cumulative_balance} separates
cross-step anisotropy $\rho$ from the relative accumulated
orthogonalization error $\tau_K$.
The condition $\rho+\tau_K<1$ ensures a positive lower
bound on the smallest squared singular value.
When $\tau_K=0$, the bounds reduce to those in
Theorem~\ref{thm:cumulative_balance}.
Proofs are provided in the appendix.

\begin{table*}[t]
	\centering
	\caption{\small Experimental results for 10 incremental tasks on ImageNetR, ImageNetA, CIFAR100, and CUB. We report the averaged results over 3 random seeds. Best and second-best results are \textbf{bold} and \underline{underlined}.}
	\vspace{-1mm}
	\setlength{\tabcolsep}{4pt}
	\resizebox{\textwidth}{!}{
		\begin{tabular}{lccccccccc}
			\toprule[1pt]
			
			\multirow{2}{*}{Method}
			& \multirow{2}{*}{Venue}
			& \multicolumn{2}{c}{10S-ImageNetR}
			& \multicolumn{2}{c}{10S-ImageNetA}
			& \multicolumn{2}{c}{10S-CIFAR100}
			& \multicolumn{2}{c}{10S-CUB}\\
			&
			& $\mathcal{A}_{Last}\uparrow$
			& $\mathcal{A}_{Avg}\uparrow$
			& $\mathcal{A}_{Last}\uparrow$
			& $\mathcal{A}_{Avg}\uparrow$
			& $\mathcal{A}_{Last}\uparrow$
			& $\mathcal{A}_{Avg}\uparrow$
			& $\mathcal{A}_{Last}\uparrow$
			& $\mathcal{A}_{Avg}\uparrow$\\
			\hline
			
			L2P
			& CVPR'22
			& $72.34_{\pm0.17}$
			& $77.36_{\pm0.64}$
			& $44.04_{\pm0.93}$
			& $51.24_{\pm2.26}$
			& $84.06_{\pm0.88}$
			& $88.26_{\pm1.34}$
			& $67.02_{\pm1.90}$
			& $79.62_{\pm1.60}$\\
			
			DualPrompt
			& ECCV'22
			& $69.10_{\pm0.62}$
			& $74.28_{\pm0.66}$
			& $53.19_{\pm0.74}$
			& $64.59_{\pm0.08}$
			& $86.93_{\pm0.24}$
			& $91.13_{\pm0.32}$
			& $68.48_{\pm0.47}$
			& $80.59_{\pm1.50}$\\
			
			CODAPrompt
			& CVPR'23
			& $73.31_{\pm0.50}$
			& $78.47_{\pm0.53}$
			& $52.08_{\pm0.12}$
			& $63.92_{\pm0.12}$
			& $83.21_{\pm3.39}$
			& $87.71_{\pm3.17}$
			& $77.23_{\pm1.12}$
			& $81.90_{\pm0.85}$\\
			
			InfLoRA
			& CVPR'24
			& $74.75_{\pm0.64}$
			& $80.67_{\pm0.55}$
			& $49.20_{\pm1.12}$
			& $60.92_{\pm0.61}$
			& $86.75_{\pm0.35}$
			& $91.72_{\pm0.15}$
			& $70.82_{\pm0.23}$
			& $81.39_{\pm0.14}$\\
			
			SD-LoRA
			& ICLR'25
			& $77.34_{\pm0.35}$
			& $82.04_{\pm0.24}$
			& $55.96_{\pm0.73}$
			& $64.95_{\pm1.63}$
			& $88.01_{\pm0.31}$
			& $92.54_{\pm0.18}$
			& $77.48_{\pm0.20}$
			& $85.59_{\pm0.44}$\\
			
			Bi-LoRA
			& CVPR'25
			& $77.95_{\pm0.14}$
			& $81.52_{\pm0.26}$
			& --
			& --
			& $87.46_{\pm0.76}$
			& $92.50_{\pm0.62}$
			& --
			& --\\
			
			PLAN
			& ICCV'25
			& $75.25_{\pm0.42}$
			& $80.41_{\pm0.56}$
			& --
			& --
			& $87.54_{\pm0.31}$
			& $92.21_{\pm0.35}$
			& --
			& --\\
			
			LoRA-P\&M
			& NeurIPS'25
			& $79.95_{\pm0.18}$
			& $85.29_{\pm0.93}$
			& $56.57_{\pm0.78}$
			& $65.35_{\pm1.81}$
			& $88.45_{\pm0.35}$
			& $92.89_{\pm1.13}$
			& $78.29_{\pm0.50}$
			& $83.39_{\pm0.61}$\\
			
			CoSO
			& NeurIPS'25
			& $\underline{81.10}_{\pm0.39}$
			& $85.56_{\pm0.13}$
			& --
			& --
			& $88.77_{\pm0.16}$
			& $92.99_{\pm0.23}$
			& --
			& --\\
			
			
			SplitLoRA
			& ICLR'26
			& $81.00_{\pm0.17}$
			& $\underline{85.84}_{\pm0.62}$
			& --
			& --
			& $\mathbf{90.33_{\pm0.73}}$
			& $\underline{93.70}_{\pm0.32}$
			& --
			& --\\
			
			EBLoRA
			& ICML'26
			& $79.56_{\pm0.68}$
			& $82.68_{\pm0.86}$
			& $\underline{61.56}_{\pm0.54}$
			& $\underline{69.28}_{\pm0.56}$
			& $88.98_{\pm0.61}$
			& $92.86_{\pm0.42}$
			& $\underline{78.92}_{\pm0.38}$
			& $\underline{88.64}_{\pm0.52}$\\
			
			EWC-LoRA
			& ICLR'26
			& $72.86_{\pm0.79}$
			& $78.95_{\pm0.86}$
			& $59.89_{\pm0.26}$
			& $68.33_{\pm0.67}$
			& $87.91_{\pm0.57}$
			& $92.27_{\pm0.39}$
			& $72.62_{\pm0.32}$
			& $83.52_{\pm0.54}$\\
			
			\midrule
			\texttt{MLoRA (Ours)}
			& --
			& $\mathbf{81.77_{\pm0.25}}$
			& $\mathbf{87.28_{\pm0.28}}$
			& $\mathbf{62.23_{\pm0.52}}$
			& $\mathbf{71.27_{\pm1.02}}$
			& $\underline{90.03}_{\pm0.26}$
			& $\mathbf{93.81_{\pm1.42}}$
			& $\mathbf{80.27_{\pm0.25}}$
			& $\mathbf{89.47_{\pm0.42}}$\\
			
			\bottomrule[1pt]
		\end{tabular}
	}
	\vspace{-1mm}
	\label{tab:sota-tab1}
\end{table*}
\begin{table*}[t]
	\centering
	\caption{\small Experimental results for 5 and 20 incremental sessions on ImageNetR, 20 sessions on ImageNetA, and 10 sessions on DomainNet. The highest results are in \textbf{bold}, and the second highest results are \underline{underlined}.}
	\setlength{\tabcolsep}{4pt}
	\resizebox{\textwidth}{!}{
		\begin{tabular}{lccccccccc}
			\toprule[1pt]
			
			\multirow{2}{*}{Method}
			& \multirow{2}{*}{Venue}
			& \multicolumn{2}{c}{5S-ImageNetR}
			& \multicolumn{2}{c}{20S-ImageNetR}
			& \multicolumn{2}{c}{20S-ImageNetA}
			& \multicolumn{2}{c}{10S-DomainNet}\\
			&
			& $\mathcal{A}_{Last}\uparrow$
			& $\mathcal{A}_{Avg}\uparrow$
			& $\mathcal{A}_{Last}\uparrow$
			& $\mathcal{A}_{Avg}\uparrow$
			& $\mathcal{A}_{Last}\uparrow$
			& $\mathcal{A}_{Avg}\uparrow$
			& $\mathcal{A}_{Last}\uparrow$
			& $\mathcal{A}_{Avg}\uparrow$\\
			\hline
			
			L2P & CVPR'22
			& $70.83_{\pm0.58}$ & $78.34_{\pm0.47}$
			& $69.64_{\pm0.42}$ & $75.28_{\pm0.57}$
			& $40.48_{\pm1.78}$ & $49.62_{\pm1.46}$
			& $81.17_{\pm0.83}$ & $87.43_{\pm0.95}$\\
			
			DualPrompt & ECCV'22
			& $73.05_{\pm0.50}$ & $79.47_{\pm0.40}$
			& $66.61_{\pm0.58}$ & $72.45_{\pm0.37}$
			& $42.28_{\pm1.94}$ & $53.39_{\pm1.64}$
			& $81.70_{\pm0.78}$ & $87.80_{\pm0.99}$\\
			
			CODAPrompt & CVPR'23
			& $74.91_{\pm0.33}$ & $79.25_{\pm0.53}$
			& $69.96_{\pm0.50}$ & $75.34_{\pm0.85}$
			& $44.62_{\pm1.92}$ & $54.86_{\pm0.50}$
			& $80.04_{\pm0.79}$ & $86.27_{\pm0.82}$\\
			
			InfLoRA & CVPR'24
			& $77.52_{\pm0.37}$ & $82.01_{\pm0.12}$
			& $71.01_{\pm0.45}$ & $77.28_{\pm0.45}$
			& $46.81_{\pm2.30}$ & $57.13_{\pm0.95}$
			& $81.45_{\pm0.68}$ & $87.55_{\pm0.57}$\\
			
			SD-LoRA & ICLR'25
			& $79.15_{\pm0.20}$ & $83.01_{\pm0.42}$
			& $75.26_{\pm0.37}$ & $80.22_{\pm0.72}$
			& $49.12_{\pm1.78}$ & $59.63_{\pm1.57}$
			& $81.01_{\pm0.42}$ & $86.85_{\pm0.40}$\\
			
			Bi-LoRA & CVPR'25
			& $78.32_{\pm0.37}$ & $83.31_{\pm0.52}$
			& $72.41_{\pm0.65}$ & $79.28_{\pm0.76}$
			& -- & --
			& $76.56_{\pm0.41}$ & $83.20_{\pm0.35}$\\
			
			LoRA-P\&M & NeurIPS'25
			& $81.47_{\pm0.56}$ & $85.96_{\pm0.52}$
			& $76.37_{\pm0.09}$ & $82.77_{\pm0.31}$
			& $52.27_{\pm1.84}$ & $62.15_{\pm1.33}$
			& $\underline{84.35}_{\pm0.48}$
			& $\underline{89.43}_{\pm0.27}$\\
			
			CoSO & NeurIPS'25
			& $82.10_{\pm0.13}$
			& $\underline{86.38}_{\pm0.07}$
			& $78.19_{\pm0.28}$ & $83.69_{\pm0.12}$
			& -- & --
			& -- & --\\
			
			
			SplitLoRA & ICLR'26
			& $81.92_{\pm0.29}$ & $85.85_{\pm0.55}$
			& $\underline{78.82}_{\pm0.28}$
			& $\underline{84.57}_{\pm0.44}$
			& -- & --
			& $84.31_{\pm0.23}$ & $88.99_{\pm0.57}$\\
			
			EBLoRA & ICML'26
			& $\underline{82.56}_{\pm0.32}$
			& $86.11_{\pm0.25}$
			& $77.45_{\pm0.23}$ & $83.76_{\pm0.65}$
			& $\underline{52.52}_{\pm1.01}$
			& $\underline{62.16}_{\pm0.82}$
			& $83.67_{\pm0.14}$ & $89.27_{\pm0.21}$\\
			
			EWC-LoRA & ICLR'26
			& $76.36_{\pm0.25}$ & $82.12_{\pm0.28}$
			& $70.18_{\pm0.15}$ & $78.56_{\pm0.74}$
			& $48.58_{\pm1.15}$ & $59.06_{\pm0.82}$
			& $83.56_{\pm0.28}$ & $89.25_{\pm0.12}$\\
			
			\midrule
			\texttt{MLoRA (Ours)} & --
			& $\mathbf{83.21}_{\pm0.24}$
			& $\mathbf{87.56}_{\pm0.32}$
			& $\mathbf{78.93}_{\pm0.18}$
			& $\mathbf{85.02}_{\pm0.21}$
			& $\mathbf{55.86}_{\pm0.61}$
			& $\mathbf{65.61}_{\pm0.64}$
			& $\mathbf{85.61}_{\pm0.10}$
			& $\mathbf{90.21}_{\pm0.32}$\\
			
			\bottomrule[1pt]
		\end{tabular}
	}
	\label{tab:sota-tab2}
	\vspace{-0.5mm}
\end{table*}
\section{Experiments}

\begin{table*}[t]
	\centering
	\caption{Ablation studies on ImageNetR and ImageNetA under
		10 and 20 incremental sessions, evaluating the individual
		and combined contributions of Relative-Response Spectral Allocation (RRS) and Subspace Spectral Balancing (SSB). Idx 1 uses a fixed random orthonormal $A$ and trains only $B$.}
	\label{tab:ablation-study}
	\setlength{\tabcolsep}{4pt}
	\vspace{-1mm}
	\resizebox{\textwidth}{!}{
		\begin{tabular}{ccccccccccc}
			\toprule[1pt]
			\multirow{2}{*}{Idx} &
			\multirow{2}{*}{RRS} &
			\multirow{2}{*}{SSB} &
			\multicolumn{2}{c}{10S-ImageNetR} &
			\multicolumn{2}{c}{10S-ImageNetA} &
			\multicolumn{2}{c}{20S-ImageNetR} &
			\multicolumn{2}{c}{20S-ImageNetA} \\
			\cmidrule(lr){4-5}
			\cmidrule(lr){6-7}
			\cmidrule(lr){8-9}
			\cmidrule(lr){10-11}
			& & &
			$A_{Last}\uparrow$ & $A_{Avg}\uparrow$ &
			$A_{Last}\uparrow$ & $A_{Avg}\uparrow$ &
			$A_{Last}\uparrow$ & $A_{Avg}\uparrow$ &
			$A_{Last}\uparrow$ & $A_{Avg}\uparrow$ \\
			\midrule
			1 & -- & -- &
			 $76.23_{\pm 0.18}$ & $82.05_{\pm 0.62}$
& $57.48_{\pm 1.47}$ & $67.34_{\pm 1.86}$
& $76.21_{\pm 0.18}$ & $80.26_{\pm 0.51}$
& $48.62_{\pm 1.28}$ & $60.52_{\pm 1.26}$ \\
			2 & \ding{51} & -- &
			$79.23_{\pm 0.28 }$ & $85.34_{\pm0.36 }$ &
			$59.77_{\pm0.56 }$ & $68.91_{\pm1.08 }$ &
			$77.86_{\pm0.20 }$ & $83.06_{\pm 0.32}$ &
			$51.92_{\pm 0.66}$ & $62.42_{\pm 0.51}$ \\
			3 & -- & \ding{51} &
			$80.74_{\pm 0.28}$ & $86.31_{\pm 0.35}$ &
			$61.86_{\pm 0.55}$ & $70.26_{\pm 0.98}$ &
			$78.50_{\pm 0.26}$ & $84.25_{\pm 0.36}$ &
			$54.32_{\pm 0.64}$ & $64.53_{\pm 0.51}$ \\
			4 & \ding{51} & \ding{51} &
			$\mathbf{81.77}_{\pm0.25}$ &
			$\mathbf{87.28}_{\pm0.28}$ &
			$\mathbf{62.23}_{\pm0.52}$ &
			$\mathbf{71.27}_{\pm1.02}$ &
			$\mathbf{78.93}_{\pm0.18}$ &
			$\mathbf{85.02}_{\pm0.21}$ &
			$\mathbf{55.86}_{\pm0.61}$ &
			$\mathbf{65.61}_{\pm0.64}$ \\
			\bottomrule[1pt]
		\end{tabular}
	}
	\vspace{-1mm}
\end{table*}

\begin{figure*}[t]
	\centering
	\captionsetup[subfigure]{font=small,labelfont=bf}
	\begin{subfigure}[t]{0.24\linewidth}
		\centering
		\includegraphics[width=\linewidth]{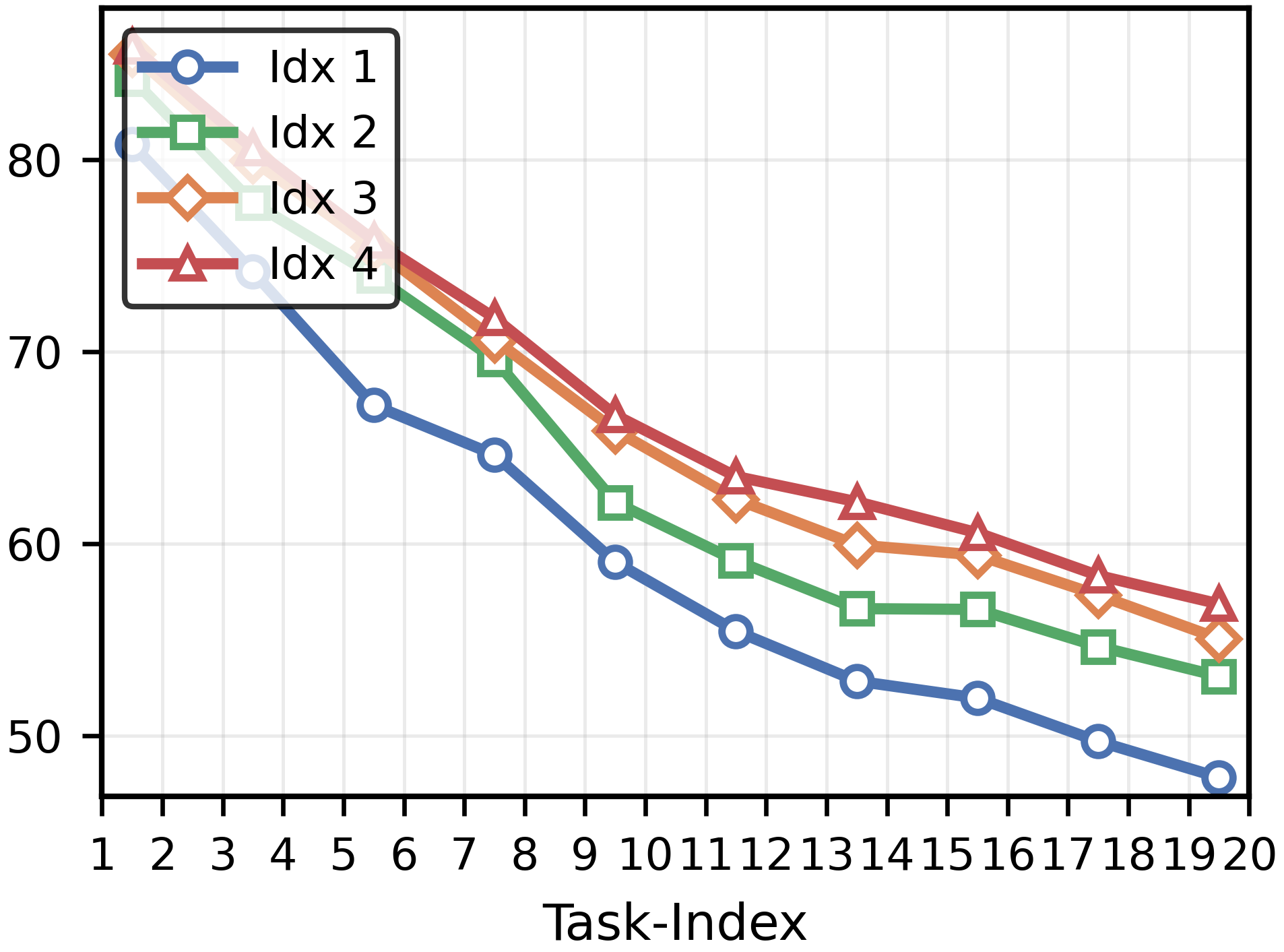}
		\caption{20S-ImageNetA (all)}
	\end{subfigure}
	\hfill
	\begin{subfigure}[t]{0.24\linewidth}
		\centering
		\includegraphics[width=\linewidth]{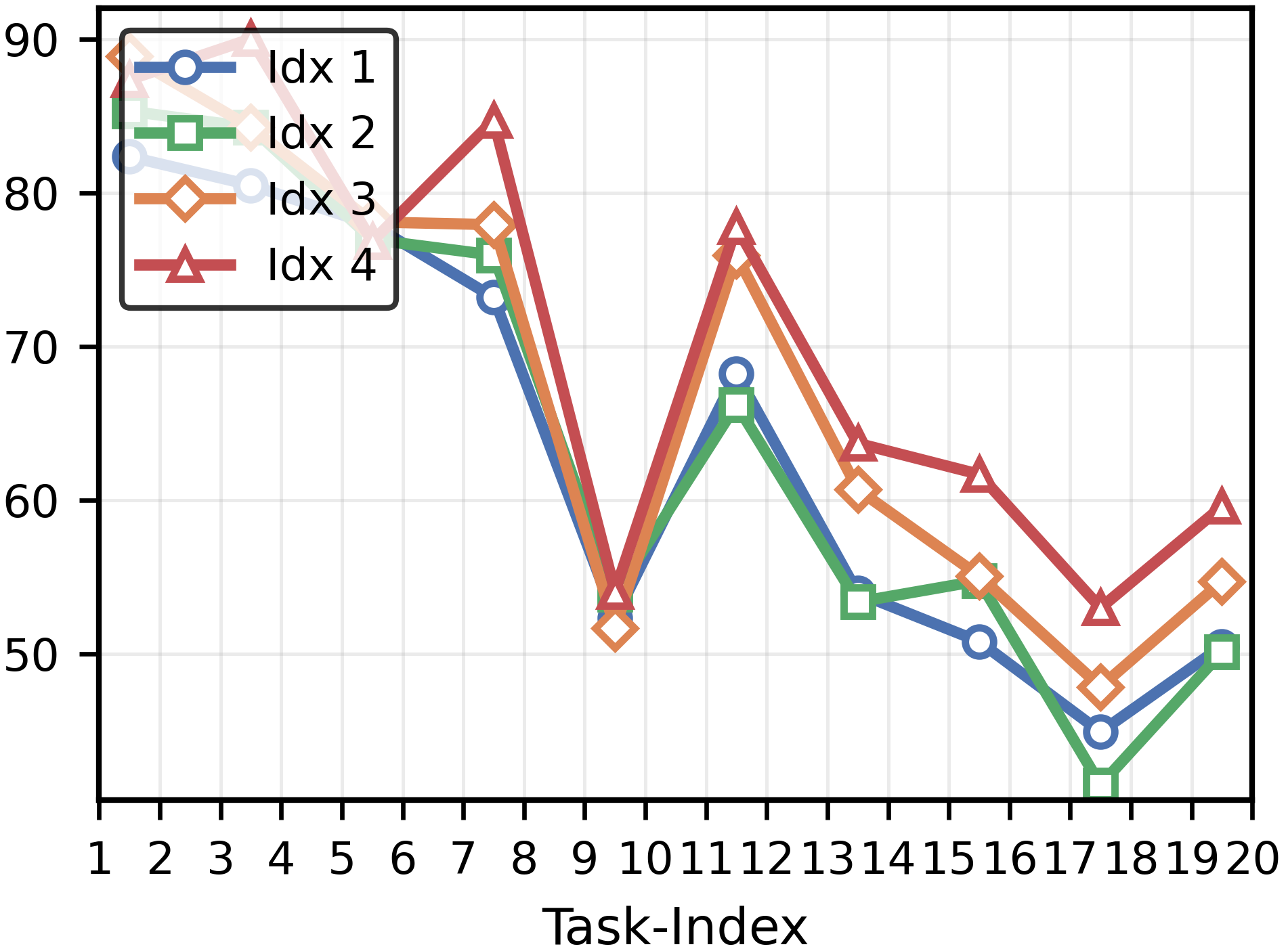}
		\caption{20S-ImageNetA (new)}
	\end{subfigure}
	\hfill
	\begin{subfigure}[t]{0.24\linewidth}
		\centering
		\includegraphics[width=\linewidth]{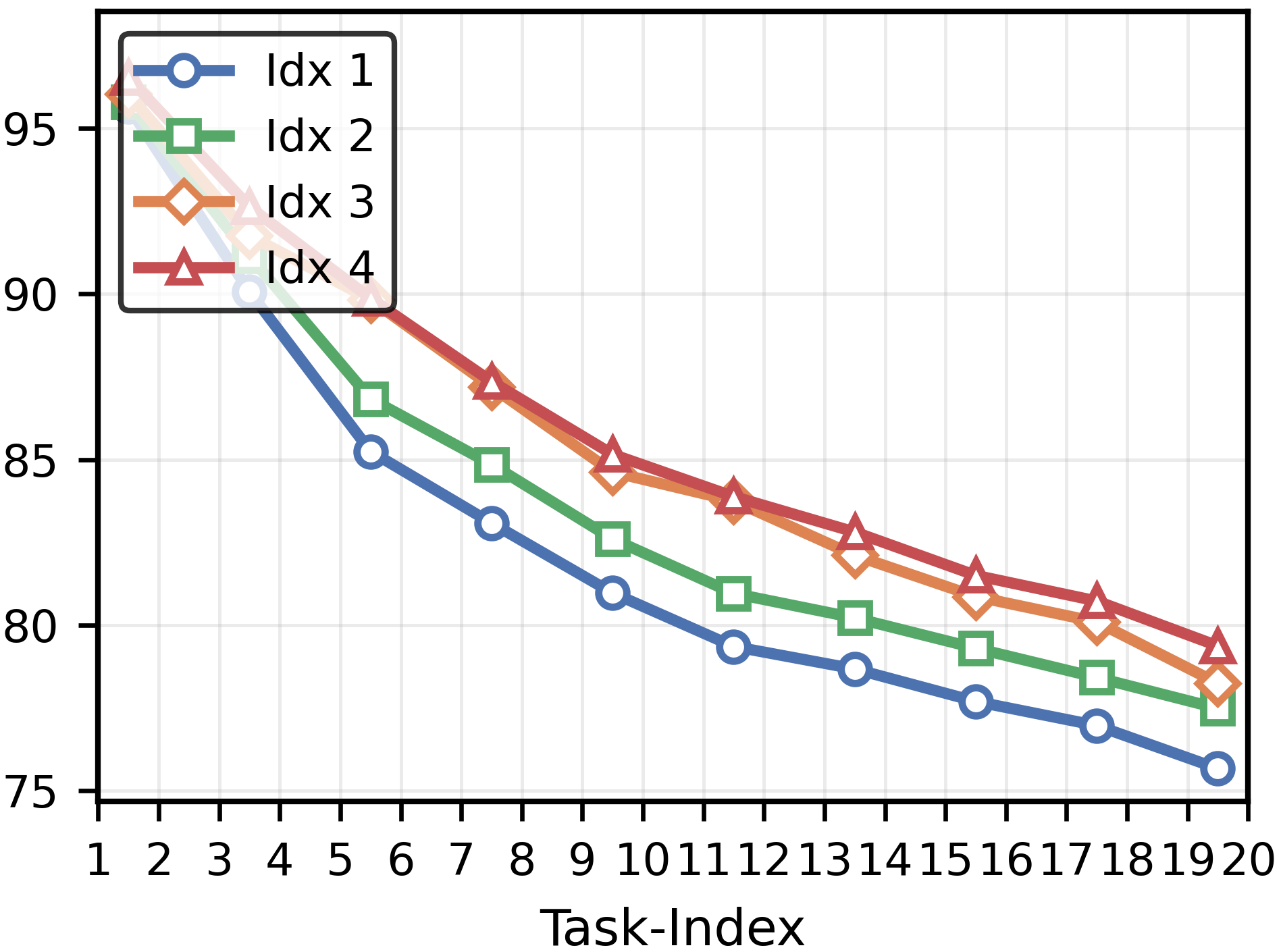}
		\caption{20S-ImageNetR (all)}
	\end{subfigure}
	\hfill
	\begin{subfigure}[t]{0.24\linewidth}
		\centering
		\includegraphics[width=\linewidth]{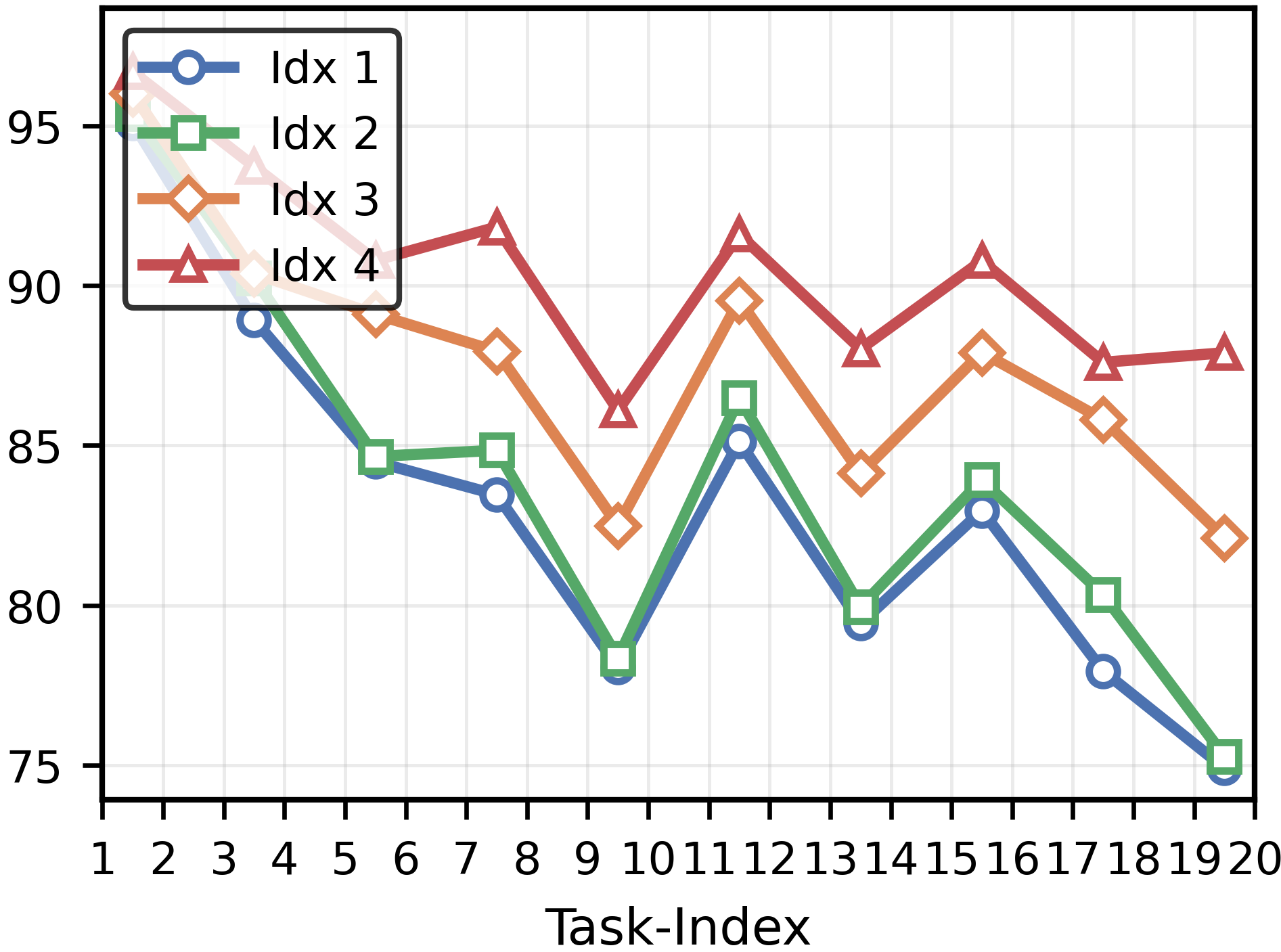}
		\caption{20S-ImageNetR (new)}
	\end{subfigure}
	\vspace{-1.0mm}
	\caption{
		Stability--plasticity dynamics across incremental tasks.
		Overall and new-task accuracy on 20S-ImageNetA and 20S-ImageNetR.
		\texttt{MLoRA} better maintains overall performance while preserving
		strong adaptation to newly introduced classes.}
	\vspace{-2.0mm}
	\label{fig:curve}
\end{figure*}

\textbf{Datasets and Evaluation Metrics.} We evaluate our method on five datasets, namely ImageNetR~\cite{imagenet-r}, ImageNetA~\cite{imagenet-a}, CIFAR-100~\cite{cifar100}, CUB~\cite{CUB} and DomainNet~\cite{domainet}. ImageNet-R contains 30,000 images from 200 categories and diverse styles. ImageNet-A includes 7,500 real-world adversarial examples, spanning 200 classes. CIFAR-100 is a standard CL benchmark, including 60,000 32×32 images from 100 classes. DomainNet consists of images from 345 classes across 6 visual domains. Following prior works~\cite{cpromptCVPR24}, we select the top 200 most populous classes for evaluation. 
Following existing PEFT-based methods~\cite{ssaitCVPR24, macilICML2025}, we report results using 2 metrics: (1) Last accuracy ($A_{{last}}$), defined as the accuracy over all classes after the final session and (2) Average accuracy ($A_{Av g}$), defined as the mean accuracy across all incremental sessions.

\textbf{Implementation Details.} All experiments are conducted on a single NVIDIA RTX 4090 GPU. We use ViT-B/16~\cite{ViT} pre-trained on ImageNet-21K~\cite{imagenet-21k} as the backbone. The model is optimized with the Muon \cite{jordan2024muon} with a Cosine Annealing scheduler. We train each task for 20 epochs with a batch size of 48. The initial learning rate is set to 0.0001 for ImageNetR, and 0.0002 for others. Following~\cite{macilICML2025}, we insert \texttt{MLoRA} into all attention blocks with rank 32, and employ a cosine classifier for  the classification task. 

\subsection{Comparison with Baseline Methods}

We compare \texttt{MLoRA} with representative prompt-based,
adapter-based, and LoRA-based continual learning
methods~\cite{l2pCVPR22, dualpromptECCV22, vqprompt-nips24,
	ssaitCVPR24, laeICCV23, infloraCVPR24, sdloraICLR25,
	planICCV25, cosoNIPS25, lorapmNIPS25}.
As shown in Tables~\ref{tab:sota-tab1}
and~\ref{tab:sota-tab2}, \texttt{MLoRA} achieves the
highest mean accuracy in 15 of 16 reported metrics
across five benchmarks and eight incremental settings.
These results support jointly addressing low-rank
capacity allocation and spectral utilization.
On 10S-ImageNetR and 10S-ImageNetA, \texttt{MLoRA}
achieves $81.77\%$ and $62.23\%$ $A_{Last}$,
respectively, exceeding the strongest reported baseline
for each setting by $0.67$ percentage points.
The corresponding $A_{Avg}$ gains are $1.44$ and $1.99$
points, indicating improvements across the learning
sequence as well as at its endpoint.
The largest gains occur on 20S-ImageNetA:
\texttt{MLoRA} reaches $55.86\%$ $A_{Last}$ and
$65.61\%$ $A_{Avg}$, improving over LoRA-P\&M
by $3.59$ and $3.46$ points, respectively.
On ImageNetR, it also obtains the highest means under
both 5- and 20-session protocols, although the
20-session $A_{Last}$ margin is modest ($0.11$ points).
The benefits extend to CUB and DomainNet, where
\texttt{MLoRA} improves $A_{Last}$ over the strongest
reported baselines by $1.35$ and $1.26$ points,
respectively.
On CIFAR-100, it achieves the highest $A_{Avg}$
($93.81\%$), while its $A_{Last}$ ($90.03\%$)
is $0.30$ points below SplitLoRA.
Overall, the results demonstrate broad empirical
benefits from combining task-responsive subspace
selection with spectrally controlled updates,
without replaying past examples or historical features.

\subsection{Ablation Studies}
\paragraph{Effectiveness of RRS and SSB.}
Table~\ref{tab:ablation-study} and Figure \ref{fig:curve} evaluate the individual
and combined effects of RRS and SSB on ImageNetR and
ImageNetA under 10- and 20-session protocols.
Both components improve performance across all four settings.
RRS alone increases $A_{Last}$ by
$1.65$--$3.30$ percentage points over the baseline,
supporting the benefit of selecting adaptation directions
based on current-task response relative to historical exposure.
SSB alone improves $A_{Last}$ by $2.29$--$5.70$ points,
yielding larger gains than RRS alone in each setting
and supporting the effectiveness of spectral balancing
within the adaptation space.
Combining RRS and SSB achieves the highest mean accuracy
on all eight reported metrics.
Compared with SSB alone, adding RRS further improves
$A_{Last}$ by $0.37$--$1.54$ points and
$A_{Avg}$ by $0.77$--$1.08$ points.
On 20S-ImageNetA, the full model reaches
$55.86\%$ $A_{Last}$ and $65.61\%$ $A_{Avg}$,
exceeding the baseline by $7.24$ and $5.09$ points,
respectively.
These results support the complementary benefits of
task-responsive subspace selection and spectral balancing:
their combination consistently outperforms either
component alone.

\section{Conclusion}

We introduced \texttt{MLoRA} to address \emph{\textbf{ spectral plasticity collapse}}
that can leave low-rank adaptation capacity underutilized
in continual learning.
\texttt{MLoRA} couples history-aware subspace selection
with Muon-style spectral balancing through a fixed
row-orthonormal basis, transferring factor-step spectra
exactly to weight space.
Our analysis establishes the max--min optimality of exact
subspace selection and provides conditional spectral
guarantees for accumulated updates.
Experiments across five benchmarks and eight incremental
settings support a central design principle:
effective adaptation requires both selecting task-responsive
directions and controlling how updates use them.

\section*{LLM Usage Statement}
The authors used OpenAI ChatGPT (Codex) to assist with manuscript editing,
including the problem definition, and with preparing the LaTeX submission
format. LLM assistance included sentence rephrasing, grammar correction,
and readability edits. The authors are responsible for verifying the
manuscript and its scientific claims.

{\small
\bibliographystyle{ieeenat_fullname}
\bibliography{aaai2027}

@inproceedings{l2pCVPR22,
	title = {Learning to prompt for continual learning},
	author = {Wang, Zifeng and Zhang, Zizhao and Lee, Chen-Yu and Zhang, Han and Sun, Ruoxi and Ren, Xiaoqi and Su, Guolong and Perot, Vincent and Dy, Jennifer and Pfister, Tomas},
	booktitle = {Proceedings of the IEEE/CVF conference on computer vision and pattern recognition},
	pages = {139--149},
	year = {2022},
}

@inproceedings{dualpromptECCV22,
	title = {Dualprompt: Complementary prompting for rehearsal-free continual learning},
	author = {Wang, Zifeng and Zhang, Zizhao and Ebrahimi, Sayna and Sun, Ruoxi and Zhang, Han and Lee, Chen-Yu and Ren, Xiaoqi and Su, Guolong and Perot, Vincent and Dy, Jennifer and others},
	booktitle = {European Conference on Computer Vision},
	pages = {631--648},
	year = {2022},
	organization = {Springer},
}

@inproceedings{codaCVPR23,
	title = {Coda-prompt: Continual decomposed attention-based prompting for rehearsal-free continual learning},
	author = {Smith, James Seale and Karlinsky, Leonid and Gutta, Vyshnavi and Cascante-Bonilla, Paola and Kim, Donghyun and Arbelle, Assaf and Panda, Rameswar and Feris, Rogerio and Kira, Zsolt},
	booktitle = {Proceedings of the IEEE/CVF Conference on Computer Vision and Pattern Recognition},
	pages = {11909--11919},
	year = {2023},
}

@inproceedings{laeICCV23,
	title = {A unified continual learning framework with general parameter-efficient tuning},
	author = {Gao, Qiankun and Zhao, Chen and Sun, Yifan and Xi, Teng and Zhang, Gang and Ghanem, Bernard and Zhang, Jian},
	booktitle = {Proceedings of the IEEE/CVF International Conference on Computer Vision},
	pages = {11483--11493},
	year = {2023},
}

@inproceedings{ssaitCVPR24,
	title = {Semantically-Shifted Incremental Adapter-Tuning is A Continual ViTransformer},
	author = {Tan, Yuwen and Zhou, Qinhao and Xiang, Xiang and Wang, Ke and Wu, Yuchuan and Li, Yongbin},
	booktitle = {Proceedings of the IEEE/CVF Conference on Computer Vision and Pattern Recognition},
	pages = {23252--23262},
	year = {2024},
}

@inproceedings{cpromptCVPR24,
	title = {Consistent Prompting for Rehearsal-Free Continual Learning},
	author = {Gao, Zhanxin and Cen, Jun and Chang, Xiaobin},
	booktitle = {Proceedings of the IEEE/CVF Conference on Computer Vision and Pattern Recognition},
	pages = {28463--28473},
	year = {2024},
}

@inproceedings{infloraCVPR24,
	title = {InfLoRA: Interference-Free Low-Rank Adaptation for Continual Learning},
	author = {Liang, Yan-Shuo and Li, Wu-Jun},
	booktitle = {Proceedings of the IEEE/CVF Conference on Computer Vision and Pattern Recognition},
	pages = {23638--23647},
	year = {2024},
}

@inproceedings{vptnspNIPS-24,
	title = {Visual {{Prompt Tuning}} in {{Null Space}} for {{Continual Learning}}},
	booktitle = {Conference on {{Neural Information Processing Systems}}},
	author = {Lu, Yue and Zhang, Shizhou and Cheng, De and Xing, Yinghui and Wang, Nannan and Wang, Peng and Zhang, Yanning},
	year = 2024,
}

@inproceedings{pgpICLR24,
	title = {Prompt Gradient Projection for Continual Learning},
	author = {Qiao, Jingyang and Tan, Xin and Chen, Chengwei and Qu, Yanyun and Peng, Yong and Xie, Yuan and others},
	booktitle = {The Twelfth International Conference on Learning Representations},
	year = {2023},
}

@article{ViT,
	title = {An Image is Worth 16x16 Words: Transformers for Image Recognition at Scale},
	author = {Dosovitskiy, Alexey and Beyer, Lucas and Kolesnikov, Alexander and Weissenborn, Dirk and Zhai, Xiaohua and Unterthiner, Thomas and  Dehghani, Mostafa and Minderer, Matthias and Heigold, Georg and Gelly, Sylvain and Uszkoreit, Jakob and Houlsby, Neil},
	journal = {ICLR},
	year = {2021},
}

@inproceedings{lora,
	title = {Lo{RA}: Low-Rank Adaptation of Large Language Models},
	author = {Edward J Hu and Yelong Shen and Phillip Wallis and Zeyuan Allen-Zhu and Yuanzhi Li and Shean Wang and Lu Wang and Weizhu Chen},
	booktitle = {International Conference on Learning Representations},
	year = {2022},
	url = {https://openreview.net/forum?id=nZeVKeeFYf9},
}

@inproceedings{imagenet-r,
	title = {The many faces of robustness: A critical analysis of out-of-distribution generalization},
	author = {Hendrycks, Dan and Basart, Steven and Mu, Norman and Kadavath, Saurav and Wang, Frank and Dorundo, Evan and Desai, Rahul and Zhu, Tyler and Parajuli, Samyak and Guo, Mike and others},
	booktitle = {Proceedings of the IEEE/CVF international conference on computer vision},
	pages = {8340--8349},
	year = {2021},
}

@inproceedings{imagenet-a,
	title = {Natural adversarial examples},
	author = {Hendrycks, Dan and Zhao, Kevin and Basart, Steven and Steinhardt, Jacob and Song, Dawn},
	booktitle = {Proceedings of the IEEE/CVF conference on computer vision and pattern recognition},
	pages = {15262--15271},
	year = {2021},
}

@article{cifar100,
	title = {Learning multiple layers of features from tiny images},
	author = {Krizhevsky, Alex and Hinton, Geoffrey and others},
	year = {2009},
	publisher = {Toronto, ON, Canada},
}

@inproceedings{domainet,
	title = {Moment matching for multi-source domain adaptation},
	author = {Peng, Xingchao and Bai, Qinxun and Xia, Xide and Huang, Zijun and Saenko, Kate and Wang, Bo},
	booktitle = {Proceedings of the IEEE/CVF international conference on computer vision},
	pages = {1406--1415},
	year = {2019},
}

@article{imagenet-21k,
	title = {Imagenet large scale visual recognition challenge},
	author = {Russakovsky, Olga and Deng, Jia and Su, Hao and Krause, Jonathan and Satheesh, Sanjeev and Ma, Sean and Huang, Zhiheng and Karpathy, Andrej and Khosla, Aditya and Bernstein, Michael and others},
	journal = {International journal of computer vision},
	volume = {115},
	pages = {211--252},
	year = {2015},
	publisher = {Springer},
}

@article{EwC,
	title = {Overcoming catastrophic forgetting in neural networks},
	author = {Kirkpatrick, James and Pascanu, Razvan and Rabinowitz, Neil and Veness, Joel and Desjardins, Guillaume and Rusu, Andrei A and Milan, Kieran and Quan, John and Ramalho, Tiago and Grabska-Barwinska, Agnieszka and others},
	journal = {Proceedings of the national academy of sciences},
	volume = {114},
	number = {13},
	pages = {3521--3526},
	year = {2017},
	publisher = {National Acad Sciences},
}

@inproceedings{2021nullspace,
	title = {Training networks in null space of feature covariance for continual learning},
	author = {Wang, Shipeng and Li, Xiaorong and Sun, Jian and Xu, Zongben},
	booktitle = {Proceedings of the IEEE/CVF conference on Computer Vision and Pattern Recognition},
	pages = {184--193},
	year = {2021},
}

@article{CUB,
	title = {The caltech-ucsd birds-200-2011 dataset},
	author = {Wah, Catherine and Branson, Steve and Welinder, Peter and Perona, Pietro and Belongie, Serge},
	year = {2011},
	publisher = {California Institute of Technology},
}

@inproceedings{macilICML2025,
	title = {Navigating Semantic Drift in Task-Agnostic Class-Incremental Learning},
	author = {Fangwen Wu and Lechao Cheng and Shengeng Tang and Xiaofeng Zhu and Chaowei Fang and Dingwen Zhang and Meng Wang},
	booktitle = {Forty-second International Conference on Machine Learning},
	year = {2025},
}

@inproceedings{2021clip,
	title = {Learning transferable visual models from natural language supervision},
	author = {Radford, Alec and Kim, Jong Wook and Hallacy, Chris and Ramesh, Aditya and Goh, Gabriel and Agarwal, Sandhini and Sastry, Girish and Askell, Amanda and Mishkin, Pamela and Clark, Jack and others},
	booktitle = {International conference on machine learning},
	pages = {8748--8763},
	year = {2021},
	organization = {PmLR},
}

@inproceedings{DIA-cvpr25,
	title = {Dynamic integration of task-specific adapters for class incremental learning},
	author = {Li, Jiashuo and Wang, Shaokun and Qian, Bo and He, Yuhang and Wei, Xing and Wang, Qiang and Gong, Yihong},
	booktitle = {Proceedings of the Computer Vision and Pattern Recognition Conference},
	pages = {30545--30555},
	year = {2025},
}

@article{vqprompt-nips24,
	title = {Vector Quantization Prompting for Continual Learning},
	author = {Jiao, Li and Lai, Qiuxia and Li, Yu and Xu, Qiang},
	journal = {NeurIPS},
	year = {2024},
}

@inproceedings{biloraCVPR25,
	title = {BiLoRA: Almost-Orthogonal Parameter Spaces for Continual Learning},
	author = {Zhu, Hao and Zhang, Yifei and Dong, Junhao and Koniusz, Piotr},
	booktitle = {Proceedings of the Computer Vision and Pattern Recognition Conference},
	pages = {25613--25622},
	year = {2025},
}

@inproceedings{planICCV25,
	title = {PLAN: Proactive Low-Rank Allocation for Continual Learning},
	author = {Wang, Xiequn and Zhuang, Zhan and Zhang, Yu},
	booktitle = {Proceedings of the IEEE/CVF International Conference on Computer Vision},
	pages = {2909--2918},
	year = {2025},
}

@inproceedings{cosoNIPS25,
	title = {Continuous Subspace Optimization for Continual Learning},
	author = {Cheng, Quan and Wan, Yuanyu and Wu, Lingyu and Hou, Chenping and Zhang, Lijun},
	booktitle = {Conference on {{Neural Information Processing Systems}}},
	year = 2025,
}

@inproceedings{lorapmNIPS25,
	title = {Train with Perturbation, Infer after Merging: A Two-Stage Framework for Continual Learning},
	booktitle = {Conference on {{Neural Information Processing Systems}}},
	author = {Qiu, Haomiao and Zhang, Miao and Qiao, Ziyue and Nie, Liqiang},
	year = 2025,
}

@inproceedings{sdloraICLR25,
	title = {{SD}-Lo{RA}: Scalable Decoupled Low-Rank Adaptation for Class Incremental Learning},
	author = {Yichen Wu and Hongming Piao and Long-Kai Huang and Renzhen Wang and Wanhua Li and Hanspeter Pfister and Deyu Meng and Kede Ma and Ying Wei},
	booktitle = {The Thirteenth International Conference on Learning Representations},
	year = {2025},
	url = {https://openreview.net/forum?id=5U1rlpX68A},
}

@inproceedings{olora-acl2023,
	title = {Orthogonal subspace learning for language model continual learning},
	author = {Wang, Xiao and Chen, Tianze and Ge, Qiming and Xia, Han and Bao, Rong and Zheng, Rui and Zhang, Qi and Gui, Tao and Huang, Xuan-Jing},
	booktitle = {Findings of the Association for Computational Linguistics: EMNLP 2023},
	pages = {10658--10671},
	year = {2023},
}

@inproceedings{cheng2024disentangled,
	title = {Disentangled prompt representation for domain generalization},
	author = {Cheng, De and Xu, Zhipeng and Jiang, Xinyang and Wang, Nannan and Li, Dongsheng and Gao, Xinbo},
	booktitle = {Proceedings of the IEEE/CVF Conference on Computer Vision and Pattern Recognition},
	pages = {23595--23604},
	year = {2024},
}

@article{wang2026ekpc,
	title = {EKPC: Elastic Knowledge Preservation and Compensation for Class-Incremental Learning: H. Wang et al.},
	author = {Wang, Huaijie and Cheng, De and He, Lingfeng and Li, Yan and Li, Jie and Wang, Nannan and Gao, Xinbo},
	journal = {International Journal of Computer Vision},
	volume = {134},
	number = {5},
	pages = {238},
	year = {2026},
	publisher = {Springer},
}

@article{he2025ckaa,
	title = {Ckaa: Cross-subspace knowledge alignment and aggregation for robust continual learning},
	author = {He, Lingfeng and Cheng, De and Ma, Zhiheng and Wang, Huaijie and Zhang, Dingwen and Wang, Nannan and Gao, Xinbo},
	journal = {arXiv preprint arXiv:2507.09471},
	year = {2025},
}

@article{qiu2025splitlora,
	title = {SplitLoRA: Balancing Stability and Plasticity in Continual Learning Through Gradient Space Splitting},
	author = {Qiu, Haomiao and Zhang, Miao and Qiao, Ziyue and Guan, Weili and Zhang, Min and Nie, Liqiang},
	journal = {arXiv preprint arXiv:2505.22370},
	year = {2025},
}

@article{halko2011finding,
	title = {Finding structure with randomness: Probabilistic algorithms for constructing approximate matrix decompositions},
	author = {Halko, Nathan and Martinsson, Per-Gunnar and Tropp, Joel A.},
	journal = {SIAM Review},
	volume = {53},
	number = {2},
	pages = {217--288},
	year = {2011},
	publisher = {SIAM},
}

@article{jordan2024muon,
	title = {Muon: An optimizer for hidden layers in neural networks, 2024},
	author = {Jordan, Keller and Jin, Yuchen and Boza, Vlado and You, Jiacheng and Cesista, Franz and Newhouse, Laker and Bernstein, Jeremy},
	journal = {URL https://kellerjordan. github. io/posts/muon},
	volume = {6},
	number = {3},
	pages = {4},
	year = {2024},
	url = {https://kellerjordan.github.io/posts/muon/},
}

@misc{bernstein2025derivingmuon,
	author = {Jeremy Bernstein},
	title = {Deriving {Muon}},
	year = {2025},
	url = {https://jeremybernste.in/writing/deriving-muon},
	note = {Technical blog post},
}

@article{gu2026spectral,
	title = {Spectral Imbalance Causes Forgetting in Low-Rank Continual Adaptation},
	author = {Gu, Hao and Luo, Mao-Lin and Zhou, Zi-Hao and
	Zhang, Han-Chen and Zhang, Min-Ling and Wei, Tong},
	journal = {arXiv preprint arXiv:2602.00722},
	year = {2026},
	doi = {10.48550/arXiv.2602.00722},
	url = {https://arxiv.org/abs/2602.00722},
}

@inproceedings{zheng2026revisiting,
	title = {Revisiting Weight Regularization for Low-Rank Continual Learning},
	author = {Zheng, Yaoyue and Zhang, Yin and van de Weijer, Joost and
	van de Ven, Gido M. and Du, Shaoyi and Zhang, Xuetao and
	Tian, Zhiqiang},
	booktitle = {International Conference on Learning Representations},
	year = {2026},
	url = {https://openreview.net/forum?id=pZj2DhfaVD},
}
}
\clearpage
\onecolumn
\appendix
\def\MuLoRACVPRSupplement{}

\ifdefined\MuLoRACVPRSupplement\else\clearpage\fi
\startcontents[appendices]
\section*{Appendix Contents}
\printcontents[appendices]{}{1}{\setcounter{tocdepth}{2}}
\clearpage

\section{Proofs of the Theoretical Results}
\label{app:mulora_proofs}

We prove the exact relative-response guarantee and the cumulative
spectral bounds stated in the main text. All matrices are real, and
all logarithms are natural. Singular values are ordered nonincreasingly.
For the cumulative analysis, $n\ge r$; all $r$ singular values of
$B\in\mathbb R^{n\times r}$ are retained, including zeros. The cumulative analysis concerns optimization steps
within one task, during which the input basis is fixed.

\subsection{Two spectral preliminaries}
\label{app:mulora_preliminaries}

\paragraph{Spectral quantities.}
For a nonzero matrix $B$ with positive singular values
$s_1\ge\cdots\ge s_q>0$, we use
\begin{equation}
	\operatorname{sr}(B)=\frac{\sum_{j=1}^q s_j^2}{s_1^2},\qquad
	\kappa_+(B)=\frac{s_1}{s_q},\qquad
	\operatorname{erank}(B)=\exp\!\left(-\sum_{j=1}^q p_j\log p_j\right),
	\quad p_j=\frac{s_j}{\sum_i s_i}.
\end{equation}
Entropy-based effective rank differs from the energy participation ratio
$(\sum_j s_j^2)^2/\sum_j s_j^4$; the next lemma relates them by inequalities.
Neither is the algebraic rank. In particular, $\kappa_+$ excludes zero
singular values and must be interpreted together with rank information.
For the zero matrix, normalized spectral statistics are undefined.

\begin{lemma}[Preservation of the nonzero singular spectrum]
	\label{lem:app_mulora_transfer}
	Let $B\in\mathbb R^{n\times r}$ and $A\in\mathbb R^{r\times m}$
	satisfy $AA^\top=I_r$. Then $BA$ and $B$ have the same nonzero
	singular values, including multiplicities, and the same rank.
	Consequently, for $B\ne0$, their stable ranks, entropy-based
	effective ranks, and nonzero condition numbers coincide.
	If $B$ has full column rank, then $\kappa_+(BA)=\kappa(B)$.
\end{lemma}

\begin{proof}
	The left Gram matrices coincide:
	\begin{equation}
		(BA)(BA)^\top=BAA^\top B^\top=BB^\top.
		\label{eq:app_mulora_left_gram}
	\end{equation}
	The positive eigenvalues of a left Gram matrix are precisely the
	squares of the nonzero singular values of the corresponding matrix.
	Thus the nonzero singular spectra, their multiplicities, and the
	ranks agree. Each stated spectral statistic is determined entirely
	by these nonzero singular values, which proves the remaining claims.
\end{proof}

\begin{lemma}[Entropy effective rank and energy participation]
	\label{lem:app_mulora_entropy}
	Let $B\ne0$ have positive singular values $s_1,\ldots,s_q$, and set
	\begin{equation}
		S_\ell=\sum_{j=1}^q s_j^\ell,\qquad
		p_j=\frac{s_j}{S_1},\qquad
		\operatorname{erank}(B)=\exp\left(-\sum_{j=1}^q p_j\log p_j\right).
	\end{equation}
	Then
	\begin{equation}
		\operatorname{erank}(B)
		\ge \frac{S_1^2}{S_2}
		\ge \frac{S_2^2}{S_4}
		=\frac{\bigl(\operatorname{tr}(B^\top B)\bigr)^2}
		{\operatorname{tr}\bigl((B^\top B)^2\bigr)}.
		\label{eq:app_mulora_entropy_participation}
	\end{equation}
	The last quantity is the participation rank of the squared singular
	values; it is generally distinct from $\operatorname{erank}(B)$.
	If all positive singular values are equal, all quantities in
	\eqref{eq:app_mulora_entropy_participation} equal $q$.
\end{lemma}

\begin{proof}
	By concavity of the logarithm, using the positive weights $p_j$,
	\begin{equation}
		\sum_{j=1}^q p_j\log p_j
		\le \log\left(\sum_{j=1}^q p_j^2\right).
	\end{equation}
	Negating and exponentiating gives
	\begin{equation}
		\operatorname{erank}(B)
		\ge \frac{1}{\sum_j p_j^2}
		=\frac{S_1^2}{S_2}.
	\end{equation}
	Next, H\"older's inequality with conjugate exponents $3/2$ and $3$
	yields
	\begin{equation}
		S_2=\sum_{j=1}^q s_j^{2/3}s_j^{4/3}
		\le \left(\sum_{j=1}^q s_j\right)^{2/3}
		\left(\sum_{j=1}^q s_j^4\right)^{1/3}
		=S_1^{2/3}S_4^{1/3}.
	\end{equation}
	Therefore $S_2^3\le S_1^2S_4$. Since $S_2,S_4>0$, division gives
	$S_1^2/S_2\ge S_2^2/S_4$.
	Finally, $\operatorname{tr}(B^\top B)=S_2$ and
	$\operatorname{tr}((B^\top B)^2)=S_4$, including when $B$ is rank
	deficient. Equal positive singular values give $p_j=1/q$ and equality
	throughout.
\end{proof}

\subsection{Exact relative-response subspace selection}
\label{app:mulora_relative_response}

This subsection proves Proposition~\ref{prop:relative_response}.

\begin{proof}
	Assume $1\le r\le m$, and abbreviate $H=H_{t-1}$, $L=L_{t-1}$,
	$M=M^t$, and $\widetilde M=L^{-1}ML^{-\top}$.
	Since $H=LL^\top\succ0$, $L$ is invertible. Moreover,
	$\widetilde M$ is symmetric positive semidefinite.
	
	\paragraph{Step 1: Transform the relative response.}
	For every $v\ne0$, set $u=L^\top v$, so that $v=L^{-\top}u$.
	Then
	\begin{equation}
		v^\top Hv=u^\top u,\qquad
		v^\top Mv=u^\top\widetilde M u.
	\end{equation}
	Consequently,
	\begin{equation}
		\frac{v^\top Mv}{v^\top Hv}
		=\frac{u^\top\widetilde M u}{u^\top u}.
		\label{eq:app_mulora_rayleigh_change}
	\end{equation}
	Thus historical whitening converts the generalized Rayleigh quotient
	into an ordinary Rayleigh quotient through an invertible change of
	coordinates.
	
	\paragraph{Step 2: Bound every direction in the selected subspace.}
	Let $\Lambda_r=\operatorname{diag}(\lambda_1,\ldots,\lambda_r)$.
	Since $U_r^\top U_r=I_r$ and $\widetilde M U_r=U_r\Lambda_r$,
	every nonzero $v\in\mathcal S_t$ can be written uniquely as
	$v=L^{-\top}U_rc$ for some $c\in\mathbb R^r\setminus\{0\}$.
	Eq.~\eqref{eq:app_mulora_rayleigh_change} gives
	\begin{equation}
		\frac{v^\top Mv}{v^\top Hv}
		=\frac{c^\top\Lambda_rc}{c^\top c}
		=\frac{\sum_{j=1}^r\lambda_j c_j^2}{\sum_{j=1}^r c_j^2}
		\in[\lambda_r,\lambda_1].
		\label{eq:app_mulora_selected_response}
	\end{equation}
	Choosing $c$ to be the $r$-th coordinate vector attains the lower
	endpoint. Hence the minimum response over $\mathcal S_t\setminus\{0\}$
	is exactly $\lambda_r$.
	
	\paragraph{Step 3: Establish global max--min optimality.}
	The map $\mathcal S\mapsto L^\top\mathcal S$ is a bijection on the
	set of $r$-dimensional subspaces. Therefore
	\begin{align}
		\max_{\dim(\mathcal S)=r}\min_{v\in\mathcal S\setminus\{0\}}
		\frac{v^\top Mv}{v^\top Hv}
		&=\max_{\dim(\mathcal U)=r}\min_{u\in\mathcal U\setminus\{0\}}
		\frac{u^\top\widetilde M u}{u^\top u}
		=\lambda_r.
		\label{eq:app_mulora_maxmin}
	\end{align}
	For completeness, the last equality can also be verified directly.
	Extend the columns of $U_r$ to an orthonormal eigenbasis
	$u_1,\ldots,u_m$ of $\widetilde M$, ordered by eigenvalue.
	Every $r$-dimensional subspace $\mathcal U$ intersects
	$\operatorname{span}\{u_r,\ldots,u_m\}$ nontrivially, since their
	dimensions sum to $m+1$. A nonzero vector in the intersection has
	Rayleigh quotient at most $\lambda_r$. Thus no such $\mathcal U$
	has minimum Rayleigh quotient greater than $\lambda_r$.
	The subspace $\operatorname{range}(U_r)$ attains $\lambda_r$ by
	\eqref{eq:app_mulora_selected_response}. This is the relevant
	Courant--Fischer characterization.
	
	\paragraph{Step 4: Verify preservation under Euclidean QR.}
	The matrix $Z=L^{-\top}U_r$ has full column rank. In its thin QR
	factorization $Z=Q_tR_t$, $Q_t^\top Q_t=I_r$ and $R_t$ is
	invertible. Hence
	\begin{equation}
		\operatorname{range}(Q_t)=\operatorname{range}(Z)=\mathcal S_t.
	\end{equation}
	Defining $A_t=Q_t^\top$ therefore gives
	\begin{equation}
		A_tA_t^\top=I_r,\qquad
		\operatorname{row}(A_t)=\operatorname{range}(A_t^\top)=\mathcal S_t,
	\end{equation}
	where row vectors are identified with their transposes.
	Thus QR preserves the max--min-optimal subspace and provides the
	orthonormal parameterization required by
	Lemma~\ref{lem:app_mulora_transfer}.
\end{proof}

\paragraph{Scope of the subspace guarantee.}
The proposition concerns exact leading eigenvectors. Repeated
eigenvalues may make the optimal subspace nonunique, but do not affect
the result. With randomized SVD, QR still preserves the computed
subspace exactly; global max--min optimality is not automatic for
that approximate subspace. Also, QR need not preserve the constraint
$Z^\top HZ=I_r$: it preserves the subspace and its relative-response
guarantee, while replacing the basis by a Euclidean-orthonormal one.

\subsection{Cumulative spectral balance with exact steps}
\label{app:mulora_exact_balance}

This subsection proves Theorem~\ref{thm:cumulative_balance}.
For clarity, the quantities used below are
\begin{align}
	C_K&=\sum_{1\le i<j\le K}\eta_i\eta_j
	(D_i^\top D_j+D_j^\top D_i),\\
	a_K&=\sum_{k=1}^K\eta_k^2+\frac{\operatorname{tr}(C_K)}{r},\qquad
	E_K=C_K-\frac{\operatorname{tr}(C_K)}{r}I_r.
\end{align}
The assumptions are $B_0=0$, $\eta_k>0$, $D_k^\top D_k=I_r$,
$a_K>0$, and $\|E_K\|_2\le\rho a_K$ for $0\le\rho<1$.

\begin{proof}
	For brevity, write $B=B_K$, $C=C_K$, $a=a_K$, and $E=E_K$.
	
	\paragraph{Step 1: Expand the cumulative Gram matrix.}
	Since $B=-\sum_{k=1}^K\eta_kD_k$ and $D_k^\top D_k=I_r$,
	\begin{align}
		B^\top B
		&=\sum_{k=1}^K\eta_k^2D_k^\top D_k
		+\sum_{1\le i<j\le K}\eta_i\eta_j
		(D_i^\top D_j+D_j^\top D_i) \notag\\
		&=\left(\sum_{k=1}^K\eta_k^2\right)I_r+C
		=aI_r+E.
		\label{eq:app_mulora_exact_gram}
	\end{align}
	By definition $E$ is symmetric and $\operatorname{tr}(E)=0$.
	Taking traces gives
	\begin{equation}
		\|B\|_F^2=\operatorname{tr}(B^\top B)=ra.
		\label{eq:app_mulora_exact_trace}
	\end{equation}
	
	\paragraph{Step 2: Bound singular values and the condition number.}
	Let $e_1\ge\cdots\ge e_r$ be the eigenvalues of $E$.
	Eq.~\eqref{eq:app_mulora_exact_gram} implies
	$\sigma_j(B)^2=a+e_j$.
	Because $|e_j|\le\|E\|_2\le\rho a$,
	\begin{equation}
		a(1-\rho)\le\sigma_j(B)^2\le a(1+\rho),\qquad j=1,\ldots,r.
	\end{equation}
	The lower bound is strictly positive because $a>0$ and $\rho<1$.
	Thus $B$ has full column rank, and taking square roots proves the
	singular-value bounds. In particular,
	\begin{equation}
		\kappa(B)
		=\sqrt{\frac{a+\lambda_{\max}(E)}{a+\lambda_{\min}(E)}}
		\le\sqrt{\frac{1+\rho}{1-\rho}}.
	\end{equation}
	
	\paragraph{Step 3: Bound the stable rank.}
	Using \eqref{eq:app_mulora_exact_trace} and
	$\|B\|_2^2=a+\lambda_{\max}(E)$ yields
	\begin{equation}
		\operatorname{sr}(B)
		=\frac{ra}{a+\lambda_{\max}(E)}
		\ge\frac{r}{1+\rho}.
	\end{equation}
	
	\paragraph{Step 4: Bound the entropy-based effective rank.}
	Since $E$ is symmetric and traceless,
	\begin{align}
		\operatorname{tr}\bigl((B^\top B)^2\bigr)
		&=\operatorname{tr}\bigl((aI_r+E)^2\bigr) \notag\\
		&=ra^2+2a\operatorname{tr}(E)+\operatorname{tr}(E^2)
		=ra^2+\|E\|_F^2.
	\end{align}
	Lemma~\ref{lem:app_mulora_entropy} therefore gives
	\begin{equation}
		\operatorname{erank}(B)
		\ge\frac{(ra)^2}{ra^2+\|E\|_F^2}
		=\frac{r}{1+\|E\|_F^2/(ra^2)}.
	\end{equation}
	Finally,
	$\|E\|_F^2\le r\|E\|_2^2\le r\rho^2a^2$, so
	\begin{equation}
		\operatorname{erank}(B)\ge\frac{r}{1+\rho^2}.
	\end{equation}
	This argument bounds the entropy effective rank through energy
	participation; it does not equate the two definitions.
	
	\paragraph{Step 5: Transfer the guarantees to the weight update.}
	If $AA^\top=I_r$, Lemma~\ref{lem:app_mulora_transfer} implies
	\begin{equation}
		\kappa_+(BA)=\kappa(B),\qquad
		\operatorname{sr}(BA)=\operatorname{sr}(B),\qquad
		\operatorname{erank}(BA)=\operatorname{erank}(B).
	\end{equation}
	Thus all stated bounds transfer without spectral distortion.
\end{proof}

\subsection{Cumulative spectral balance with approximate steps}
\label{app:mulora_robust_balance}

This subsection proves Corollary~\ref{cor:approx_cumulative_balance}.
Here $\|D_k^\top D_k-I_r\|_2\le\varepsilon_k$ and
$\delta_K=\sum_{k=1}^K\eta_k^2\varepsilon_k$.
The definitions of $C_K,a_K,E_K$ are unchanged and use the actual steps.
We assume $a_K>0$ and $\|E_K\|_2\le\rho a_K$, $0\le\rho<1$;
the condition-number and rank bounds additionally require
$\delta_K<a_K(1-\rho)$.

\begin{proof}
	Write $B=B_K$, $C=C_K$, $a=a_K$, $E=E_K$, $\delta=\delta_K$,
	and, when the strict positivity condition holds, $\tau=\delta/a$.
	
	\paragraph{Step 1: Isolate the accumulated orthogonalization error.}
	Define symmetric matrices
	\begin{equation}
		F_k=D_k^\top D_k-I_r,\qquad
		F=\sum_{k=1}^K\eta_k^2F_k.
	\end{equation}
	The assumed stepwise error bounds imply
	\begin{equation}
		\|F\|_2\le\sum_{k=1}^K\eta_k^2\|F_k\|_2
		\le\sum_{k=1}^K\eta_k^2\varepsilon_k=\delta.
		\label{eq:app_mulora_error_bound}
	\end{equation}
	Expanding the Gram matrix as before now gives
	\begin{equation}
		B^\top B=aI_r+E+F.
		\label{eq:app_mulora_robust_gram}
	\end{equation}
	The cross-step matrices $C,E$ and scalar $a$ are defined using the
	actual, possibly nonorthogonal $D_k$.
	
	\paragraph{Step 2: Bound the spectrum and establish full rank.}
	For every unit vector $x\in\mathbb R^r$, symmetry and
	\eqref{eq:app_mulora_error_bound} imply
	\begin{equation}
		a-\rho a-\delta
		\le x^\top B^\top Bx
		\le a+\rho a+\delta.
	\end{equation}
	Hence every eigenvalue of $B^\top B$ satisfies
	\begin{equation}
		a(1-\rho)-\delta\le\sigma_j(B)^2
		\le a(1+\rho)+\delta,\qquad j=1,\ldots,r.
		\label{eq:app_mulora_robust_singular}
	\end{equation}
	This bound remains valid even if its lower endpoint is nonpositive.
	For the remaining claims, assume $\delta<a(1-\rho)$.
	Then $\rho+\tau<1$, and all singular values are positive. Dividing
	the extreme bounds in \eqref{eq:app_mulora_robust_singular} gives
	\begin{equation}
		\kappa(B)\le
		\sqrt{\frac{1+\rho+\tau}{1-\rho-\tau}}.
	\end{equation}
	
	\paragraph{Step 3: Bound the stable rank.}
	Because $\operatorname{tr}(E)=0$ and
	$|\operatorname{tr}(F)|\le r\|F\|_2\le r\delta$,
	\begin{equation}
		\|B\|_F^2=ra+\operatorname{tr}(F)
		\ge r(a-\delta)=ra(1-\tau).
	\end{equation}
	Combining this with
	$\|B\|_2^2\le a(1+\rho)+\delta=a(1+\rho+\tau)$ yields
	\begin{equation}
		\operatorname{sr}(B)\ge\frac{r(1-\tau)}{1+\rho+\tau}.
	\end{equation}
	
	\paragraph{Step 4: Center the Gram matrix at its actual mean.}
	In the approximate case, $a$ need not equal $\|B\|_F^2/r$.
	Define
	\begin{equation}
		f=\frac{\operatorname{tr}(F)}{r},\qquad
		\mu=a+f=\frac{\|B\|_F^2}{r},\qquad
		T=E+F-fI_r.
	\end{equation}
	Then
	\begin{equation}
		B^\top B=\mu I_r+T,\qquad
		\operatorname{tr}(T)=0,\qquad
		\mu\ge a-\delta=a(1-\tau)>0.
		\label{eq:app_mulora_centered_gram}
	\end{equation}
	To bound $T$, use the Frobenius norm. Direct expansion gives
	\begin{equation}
		\|F-fI_r\|_F^2
		=\|F\|_F^2-\frac{\bigl(\operatorname{tr}(F)\bigr)^2}{r}
		\le\|F\|_F^2\le r\delta^2.
		\label{eq:app_mulora_centered_error}
	\end{equation}
	Thus, by the triangle inequality,
	\begin{align}
		\|T\|_F
		&\le\|E\|_F+\|F-fI_r\|_F \notag\\
		&\le\sqrt r\,\rho a+\sqrt r\,\delta
		=\sqrt r\,a(\rho+\tau).
		\label{eq:app_mulora_centered_anisotropy}
	\end{align}
	The centering estimate is made in the Frobenius norm; no contraction
	of the spectral norm under trace subtraction is required.
	
	\paragraph{Step 5: Bound the entropy-based effective rank.}
	Using \eqref{eq:app_mulora_centered_gram}, symmetry and tracelessness
	give
	\begin{equation}
		\operatorname{tr}\bigl((B^\top B)^2\bigr)=r\mu^2+\|T\|_F^2.
	\end{equation}
	Lemma~\ref{lem:app_mulora_entropy} and
	\eqref{eq:app_mulora_centered_anisotropy} therefore imply
	\begin{align}
		\operatorname{erank}(B)
		&\ge\frac{(r\mu)^2}{r\mu^2+\|T\|_F^2}
		=\frac{r}{1+\|T\|_F^2/(r\mu^2)} \notag\\
		&\ge\frac{r}{1+\left(\dfrac{\rho+\tau}{1-\tau}\right)^2}.
	\end{align}
	The final inequality follows from
	$\mu\ge a(1-\tau)>0$ and
	$\|T\|_F^2\le ra^2(\rho+\tau)^2$.
	
	\paragraph{Step 6: Transfer the bounds to the weight update.}
	Apply Lemma~\ref{lem:app_mulora_transfer} to $B$ and $A$.
	Since $B$ has full column rank under the strict positivity condition,
	the condition-number, stable-rank, and entropy-effective-rank bounds
	hold unchanged for $BA$, using $\kappa_+(BA)$ for its nonzero
	condition number.
\end{proof}

\section{Approximate Computation and Scope of the Guarantees}
\label{app:mulora_computation}

\subsection{Interpreting historical response}
\label{app:mulora_history}
RRS uses $M^t=(X_t)^\top X_t$ and
$H_{t-1}=\sum_{i<t}M^i+\epsilon I$, with $\epsilon>0$.
These unnormalized statistics weight observations according to their
sample and token counts. Replacing them by per-task averages changes
the relative weighting and is a different design choice.
Because the backbone evolves across tasks, stored historical statistics
summarize features observed at their respective collection stages.
They serve as a surrogate for historical response, rather than an exact
recomputation of all past features under the current backbone.
Positive definiteness of $H_{t-1}$ ensures that the spectral selection
problem remains well defined; it does not imply that this surrogate
exactly measures historical loss or forgetting.
The method retains aggregate second-moment matrices, even though it
does not replay past examples or individual historical features.

\subsection{A computable characterization of an approximate RRS subspace}
\label{app:mulora_approx_rrs}
Let $\widetilde M=L^{-1}ML^{-\top}\succeq0$ and let
$\widehat U\in\mathbb R^{m\times r}$ have orthonormal columns, as
returned by an approximate eigensolver. Define
\begin{equation}
	\widehat{\mathcal S}=\operatorname{range}(L^{-\top}\widehat U),\qquad
	T_r=\widehat U^\top\widetilde M\widehat U,\qquad
	\theta_r=\lambda_{\min}(T_r).
\end{equation}
For $v=L^{-\top}\widehat Uc$ with $c\ne0$, the same change of
coordinates as in the exact proof yields
\begin{equation}
	\frac{v^\top Mv}{v^\top Hv}=\frac{c^\top T_rc}{c^\top c}.
\end{equation}
Consequently, the worst-direction relative response of the computed
subspace is exactly $\theta_r$. Euclidean QR of
$L^{-\top}\widehat U$ preserves this value by preserving the subspace.
If $\lambda_r$ is the $r$-th largest eigenvalue of $\widetilde M$,
max--min optimality gives $0\le\theta_r\le\lambda_r$.
This characterizes the approximate subspace without assuming exact RSVD
optimality.

A sufficient approximation certificate is available from
$\widehat M=\widehat U T_r\widehat U^\top$.
If $\|\widetilde M-\widehat M\|_2\le\zeta$, then
\begin{equation}
	\max\{0,\lambda_r-\zeta\}\le\theta_r\le\lambda_r.
	\label{eq:app_mulora_rrs_certificate}
\end{equation}
Indeed, $\widehat M$ is positive semidefinite, its $r$-th largest
eigenvalue is $\theta_r$ (including zeros), and Weyl's eigenvalue
inequality bounds its difference from $\lambda_r$ by $\zeta$.
This certificate is sufficient and may be conservative; no numerical
value of $\zeta$ is assumed here. A small invariant-subspace residual
alone does not establish that the computed subspace is the leading one.

\subsection{Interpretation and scope of the cumulative guarantees}
\label{app:mulora_scope}

\paragraph{Why a cross-step condition is necessary.}
Individually orthogonal steps can cancel along selected directions.
For example, take $n=r=2$, $K=2$, $\eta_1=\eta_2=1$, and
\begin{equation}
	D_1=\begin{pmatrix}1&0\\0&1\end{pmatrix},\qquad
	D_2=\begin{pmatrix}1&0\\0&-1\end{pmatrix}.
\end{equation}
Both satisfy $D_k^\top D_k=I_2$, but
\begin{equation}
	B_2=-(D_1+D_2)=\begin{pmatrix}-2&0\\0&0\end{pmatrix}
\end{equation}
is rank deficient. Here $a_2=2$ and
$E_2=\operatorname{diag}(2,-2)$, so $\|E_2\|_2/a_2=1$.
The strict anisotropy assumption correctly excludes this example.
The theorem therefore provides a conditional spectral guarantee,
not an unconditional consequence of stepwise orthogonalization.

\paragraph{Exact and approximate orthogonalization.}
If $\tau_K=0$, then $\delta_K=0$ and the accumulated error matrix
$F$ vanishes. The corollary then recovers the theorem's bounds in
terms of $\rho$; retaining the exact Frobenius norm of $E_K$ also
recovers its sharper effective-rank bound. A finite polynomial
orthogonalization does not by itself guarantee
$\rho+\tau_K<1$. Similarly, a rank-deficient step necessarily has
$\|D_k^\top D_k-I_r\|_2\ge1$. The strict cumulative condition must
be checked or justified for the actual sequence of updates.

\paragraph{Within-task versus across-task accumulation.}
The analysis fixes a single input basis $A$ and studies
$B_KA=-\sum_{k=1}^K\eta_kD_kA$ within that task.
It does not directly apply to the across-task sum
$\sum_{t=1}^T B_tA_t$ when the bases $A_t$ vary.
The results characterize update geometry; they do not independently
establish a bound on task loss, classification error, or forgetting.

\subsection{Checking the cumulative conditions without pairwise step storage}
\label{app:mulora_diagnostics}
The cumulative conditions can be evaluated from the actual factor steps.
Set $D_k=-\Delta B_k/\eta_k$ for an additive update
$\Delta B_k=B_k-B_{k-1}$ with $B_0=0$ and $\eta_k>0$.
All step scaling must be included consistently; if the implementation
contains additional terms, they cannot silently be omitted.
Accumulate
\begin{equation}
	S_K=\sum_{k=1}^K\eta_k^2D_k^\top D_k,\quad
	s_K=\sum_{k=1}^K\eta_k^2,\quad
	\delta_K=\sum_{k=1}^K\eta_k^2\|D_k^\top D_k-I_r\|_2.
\end{equation}
The identity
\begin{equation}
	C_K=B_K^\top B_K-S_K
\end{equation}
recovers exactly the cross-step matrix in the theorem without retaining
all pairs of steps. Compute $a_K=s_K+\operatorname{tr}(C_K)/r$
and $E_K=C_K-\operatorname{tr}(C_K)I_r/r$.
For $a_K>0$, define
\begin{equation}
	\widehat\rho_K=\frac{\|E_K\|_2}{a_K},\qquad
	\widehat\tau_K=\frac{\delta_K}{a_K}.
\end{equation}
The strict certificate holds whenever
$\widehat\rho_K+\widehat\tau_K<1$; the corollary can then be
evaluated using these quantities. Cases with $a_K\le0$ or a failed
strict inequality must be counted as uncertified rather than discarded.
Failure of this sufficient certificate does not imply poor accuracy or
rank deficiency. These expressions specify a diagnostic procedure;
the accuracy comparisons below do not themselves establish that the
certificate holds along training trajectories.

\section{Mechanism Comparisons on 20S-ImageNet-A}
\label{app:mulora_mechanisms}

\subsection{Definitions of the compared variants}
\label{app:mulora_variants}
The comparison separates choices of input subspace from the location
at which spectral control is applied. A fixed basis is held constant
within each task and reconstructed for the next task.

\paragraph{Trainable $A,B$ + Muon.}
Both LoRA factors are optimized using Muon. The induced weight step is
$\Delta B A+B\Delta A+\Delta B\Delta A$, so controlling the spectra
of the two factor steps separately does not ensure spectral control
of their combined weight step. This variant also has more trainable
adapter parameters: $r(m+n)$ rather than $nr$ per adapted matrix.

\paragraph{Fixed random orthonormal $A$ + Muon.}
A random row-orthonormal basis satisfies $AA^\top=I_r$ and only $B$
is trained. This retains exact factor-to-weight spectral transfer
without data-dependent subspace selection.

\paragraph{Fixed current-task PCA $A$ + Muon.}
Let $U_r$ contain the leading $r$ eigenvectors of $M^t$ and set
$A=U_r^\top$. This selects the strongest current-task second-moment
directions without historical information. Here ``PCA'' denotes
uncentered second-moment spectral selection, consistent with $M^t$;
centering the features would define a different variant.

\paragraph{Fixed historical-isolation $A$ + Muon.}
Let $U_H$ be an orthonormal basis for the retained historical feature
subspace, and let $Q_H$ span its orthogonal complement. If
$V_r$ contains the leading $r$ eigenvectors of
$Q_H^\top M^tQ_H$, set $A=(Q_HV_r)^\top$.
This maximizes current-task response subject to hard historical
isolation and requires $\dim(\operatorname{range}(Q_H))\ge r$.
The retained historical dimension determines the feasible space.

\paragraph{EBLoRA.}
EBLoRA constrains the adapter's nonzero singular values through
orthonormal factors and a shared scalar magnitude. The distinction
is the object of spectral control: an adapter-spectrum constraint
versus MLoRA's control of individual momentum steps.

\paragraph{MLoRA (RRS + SSB).}
RRS selects directions by current-task response relative to historical
response, followed by Euclidean orthonormalization. SSB updates only
$B$ within this fixed basis. It does not require the accumulated
adapter to have equal singular values, and thus retains spectral
freedom across optimization steps.

\subsection{Results and interpretation}
\label{app:mulora_mechanism_results}
\begin{table}[t]
	\centering
	\caption{
		Comparison of subspace construction and spectral control
		on ImageNet-A with 20 incremental sessions.
		$A_{\mathrm{Last}}$ and $A_{\mathrm{Avg}}$ are reported
		in \%.
	}
	\label{tab:imageneta20_mechanism}
	\small
	\setlength{\tabcolsep}{5pt}
	\renewcommand{\arraystretch}{1.1}
	\begin{tabular}{lcc}
		\toprule
		Method
		& $A_{\mathrm{Last}}\uparrow$
		& $A_{\mathrm{Avg}}\uparrow$ \\
		\midrule
		Trainable $A,B$ + Muon
		& $51.34 \pm 0.35$ & $62.46 \pm 0.33$ \\
		Fixed random orthonormal $A$ + Muon
		& $54.32 \pm 0.64$ & $64.53 \pm 0.51$ \\
		Fixed current-task PCA $A$ + Muon
		& $52.45 \pm 0.34$ & $63.56 \pm 0.67$ \\
		Fixed historical-isolation $A$ + Muon
		& $54.67 \pm 0.25$ & $64.76 \pm 0.47$ \\
		EBLoRA 
		& $52.52 \pm 1.01$ & $62.16 \pm 0.82$ \\
		\midrule
		\textbf{MLoRA (RRS + SSB)}
		& $\mathbf{55.86} \pm 0.61$ & $\mathbf{65.61} \pm 0.64$ \\
		\bottomrule
	\end{tabular}
\end{table}

Table~\ref{tab:imageneta20_mechanism} reports the mechanism comparison.
MLoRA attains the highest mean final and average accuracy.
Relative to the fixed random orthonormal basis, RRS improves
$A_{\mathrm{Last}}$ and $A_{\mathrm{Avg}}$ by 1.54 and 1.08
percentage points, respectively. Its gains over current-task PCA
are 3.41 and 2.05 points, suggesting that strong current-task feature
response alone is insufficient to select an effective continual
adaptation space in this setting. Relative to historical isolation,
the gains are 1.19 and 0.85 points, supporting graded relative-response
selection over strictly excluding retained historical directions.

MLoRA exceeds jointly trainable factors with Muon by 4.52 and 3.15
points and EBLoRA by 3.34 and 3.45 points. These comparisons support
the effectiveness of the complete design. They do not establish that
fixed orthonormal parameterization is necessary, or isolate a causal
effect of spectral balancing: trainable-factor counts, optimization
dynamics, and subspace choice can differ across variants.
Together with the main-text RRS/SSB component ablation, the results
support combining history-aware subspace selection and stepwise
spectral balancing. The reported mean differences are not claims of
statistical significance; paired runs are needed to quantify their
uncertainty. Overall and average accuracy alone do not identify whether
a gain arises from improved new-task learning or reduced forgetting.

\section{Reproducibility Details and Limitations}
\label{app:mulora_reproducibility}

\subsection{Reported experimental configuration}
\label{app:mulora_configuration}
The main experimental configuration uses a ViT-B/16 backbone pretrained
on ImageNet-21K, adapter rank $r=32$, a cosine classifier, and a
single NVIDIA RTX 4090 GPU. Each task is trained for 20 epochs with
batch size 48 and a cosine learning-rate schedule. The reported
initial learning rate is $10^{-4}$ on ImageNet-R and $2\times10^{-4}$
on the other benchmarks, including ImageNet-A. The main comparison
reports averages over three random seeds.

\subsection{Memory and computational accounting}
\label{app:mulora_cost}
For an adapted matrix of size $n\times m$, the trainable factor
$B\in\mathbb R^{n\times r}$ has $nr$ parameters, while the fixed
basis requires $mr$ stored entries during a task. A dense historical
second-moment matrix requires $O(m^2)$ persistent storage per distinct
tracked input representation. Thus absence of replay does not imply
absence of historical state. Statistics may be shared only where the
corresponding input features are identical.

For $N_t$ collected feature vectors, directly forming $M^t$ costs
$O(N_tm^2)$ operations and can be accumulated in batches without
retaining all feature vectors. A dense Cholesky factorization costs
$O(m^3)$; triangular solves, approximate eigenspace extraction, and QR
contribute additional task-initialization costs. For $B\in\mathbb R^{n\times r}$,
each polynomial orthogonalization iteration based on $Y^\top Y$
costs $O(nr^2+r^3)$, with $r\le n$.
These operation counts do not establish a wall-clock advantage.
End-to-end efficiency depends on feature collection and subspace
construction as well as optimizer steps.

\subsection{Scope of empirical and theoretical conclusions}
\label{app:mulora_limitations}
The spectral-transfer identity is exact for a row-orthonormal basis,
whereas the implemented polynomial orthogonalization and eigenspace
selection are approximate. Cumulative guarantees depend on measured
or otherwise justified cross-step conditions, and apply within a
task with a fixed basis. They are not bounds on cross-task forgetting,
classification error, or generalization.

The mechanism comparison covers one 20-session benchmark setting.
Broader main-text accuracy results do not replace matched mechanism
comparisons across additional settings. A norm-matched momentum
control would help distinguish update-scale effects from spectral
reshaping. Separate old-class and new-class measurements, with the
prediction label space specified, would further distinguish retention
from adaptation. Spectral plots should identify the update matrix,
layer aggregation, normalization, and numerical rank threshold;
$\kappa_+$ should not be interpreted without rank-sensitive statistics.

\stopcontents[appendices]

\end{document}